%% file: main.tex
\documentclass[letterpaper]{article}
\usepackage{aaai}
\usepackage{times}  
\usepackage{helvet}  
\usepackage{courier}  
\usepackage[hyphens]{url}  
\usepackage{graphicx} 
\usepackage{natbib}  
\usepackage{caption} 
\usepackage{amsthm}
\usepackage{amsmath}
\usepackage{amssymb}

\theoremstyle{plain}
\newtheorem{theorem}{Theorem}[section]

\theoremstyle{definition}

\theoremstyle{remark}

\usepackage{algorithm}
\usepackage{algorithmic}

\usepackage{pgfplots}
\pgfplotsset{compat=1.18} 
\usepackage{multirow}
\usepackage{colortbl}
\usepackage{color}
\usepackage{subcaption} 
\usepackage{booktabs} 
\usepackage[capitalize]{cleveref}
\crefname{figure}{Figure}{Figure}
\usepackage{arydshln}
\usepackage{tikz}
\usepackage{tabularx}
\usepackage{longtable}
\usepackage{amsfonts}

\nocopyright

\title{Auxiliary-Hyperparameter-Free Sampling: Entropy Equilibrium for Text Generation}

\author{
Xiaodong Cai$^{1,3}$\thanks{Equal contribution. Work done during the internship of Cai at Meituan.},
Hai Lin$^{1,5}$\footnotemark[1],
Shaoxiong Zhan$^{1}$,
Weiqi Luo$^{2}$,
Hong-Gee Kim$^{4}$,
Hongyan Hao$^{3}$,
Yu Yang$^{3}$\thanks{Corresponding authors.},
Hai-Tao Zheng$^{1}$\footnotemark[2]
}

\affiliations{
$^{1}$Shenzhen International Graduate School, Tsinghua University, China\\
$^{2}$ETH Zurich, Switzerland\\
$^{3}$Meituan, China\\
$^{4}$Seoul National University, South Korea\\
$^{5}$Pengcheng Laboratory, Shenzhen, China\\
\texttt{cxd23@mails.tsinghua.edu.cn, zheng.haitao@sz.tsinghua.edu.cn, yuyang.csai@gmail.com}
}

\begin{document}
 
%

\maketitle
\begin{abstract}
Token sampling strategies critically influence text generation quality in large language models (LLMs). However, existing methods introduce additional hyperparameters, requiring extensive tuning and complicating deployment. We present Entropy Equilibrium Sampling (EES), an auxiliary hyperparameter-free approach inspired by information theory that can dynamically adjust candidate sets by balancing normalized entropy with probability mass. We evaluate EES on both reasoning and generation tasks across a range of model architectures. Our results show that EES consistently performs well across temperature settings, delivering competitive accuracy and coherence while maintaining diversity. By eliminating the need for hyperparameter tuning, EES greatly simplifies deployment while improving performance. EES is released at \url{https://github.com/shuanncai/EES}.
\end{abstract}
\input{sections/introduction}
\input{sections/related_work}
\input{sections/preliminaries}

\input{sections/method}
\input{sections/exp}

\input{sections/conclusion}

\bibliography{aaai}

\newpage
\appendix
\input{sections/appendix}

\end{document}

%% file: sections/introduction.tex
\section{Introduction}

The rapid advancement of large language models (LLMs) has revolutionized natural language generation, enabling diverse applications ranging from creative writing to scientific reasoning~\cite{brown2020language,chowdhery2023palm,touvron2023llama}. As LLMs continue to improve, the quality of generated text depends not only on model architecture and training data, but also greatly on the sampling methods employed during inference~\cite{qin2025dynamic,laskar2024systematic,Xing2024MitigatingCF,zhan2025mathsmith,zhan2025lexsembridge}.

Modern text generation relies heavily on stochastic sampling methods to balance two fundamental objectives: maintaining coherence with the learned distribution while introducing sufficient diversity to avoid repetitive or overly predictable outputs~\cite{nguyen2024turning,fan2018hierarchical,welleck2019neural}. This coherence-diversity trade-off has led to the development of numerous sampling techniques, including nucleus sampling (top-p)~\cite{holtzman2019curious}, typical sampling~\cite{meister2023locally}, and eta sampling~\cite{hewitt2022truncation}. Each method introduces specialized auxiliary hyperparameters to control generation behavior~\cite{welleck2024decoding}.

\begin{figure}[tb]
\centering

\includegraphics[width=\columnwidth]{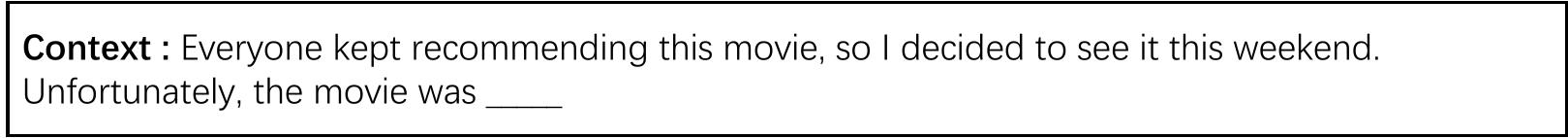}

\vspace{0cm} 

\begin{subfigure}{0.48\columnwidth}
    \centering
    \includegraphics[width=\textwidth]{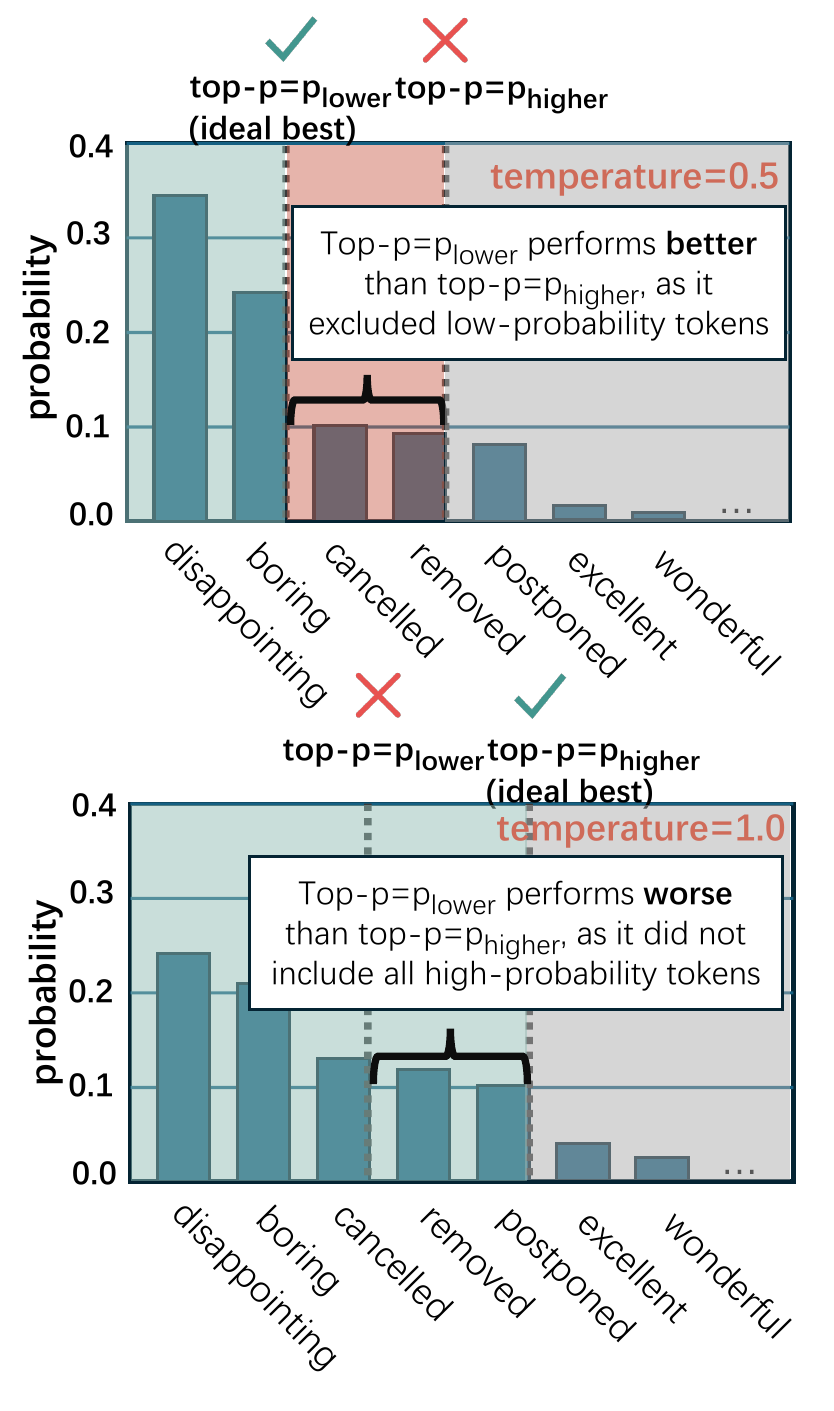}
    \vspace{-0.8cm}
    \caption{top-p sampling at $\tau$=0.5 and $\tau$=1.0}
\end{subfigure}
\hspace{0.01\columnwidth}%
\begin{minipage}[c]{0.01\columnwidth}
    \centering
    \raisebox{1.5cm}[0pt][0pt]{\rule{0.5pt}{5cm}}
\end{minipage}%
\hspace{0.01\columnwidth}%
\begin{subfigure}{0.48\columnwidth}
    \centering
    \includegraphics[width=\textwidth]{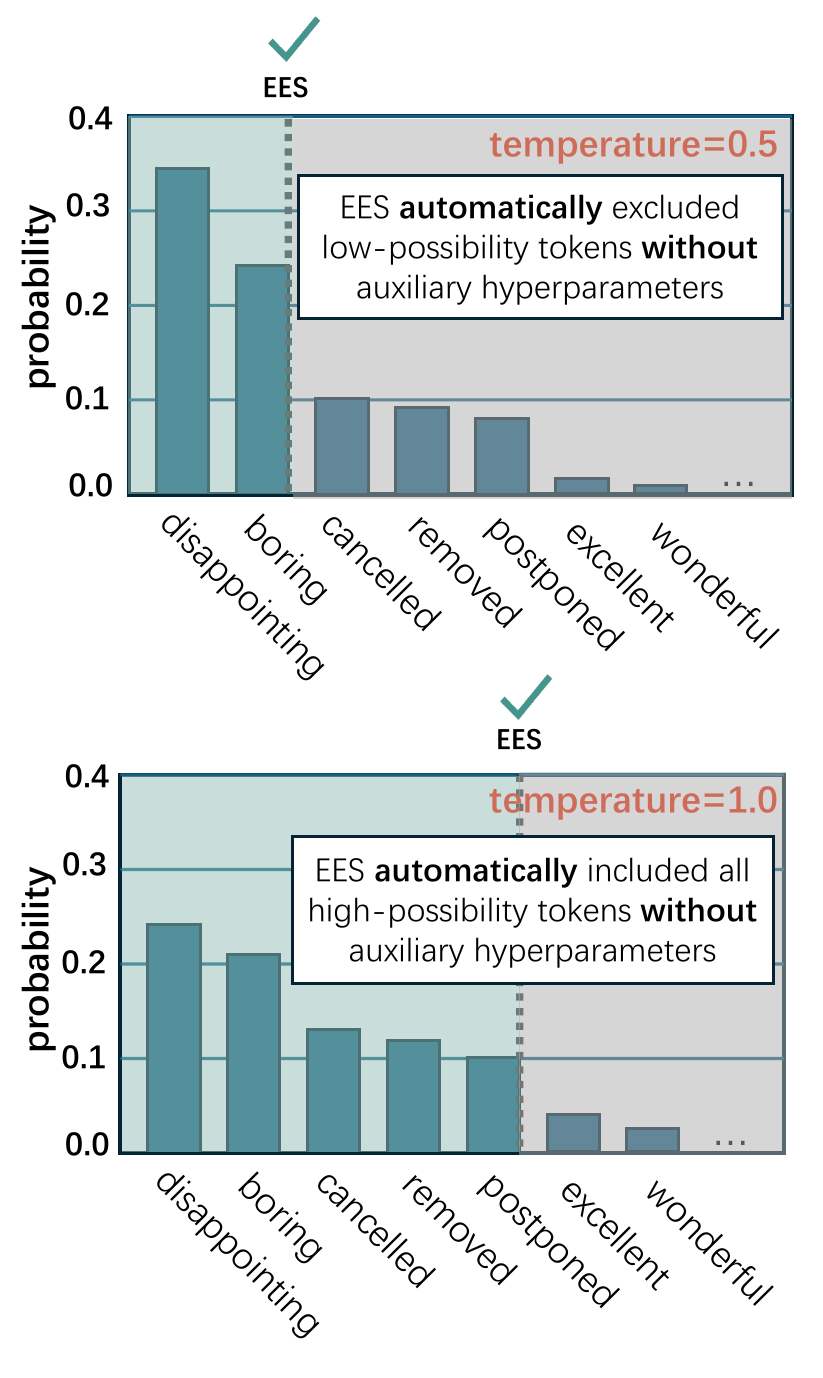}
    \vspace{-0.8cm}
    \caption{EES sampling at $\tau$=0.5 and $\tau$=1.0}
\end{subfigure}

\vspace{0cm}

\includegraphics[width=\columnwidth]{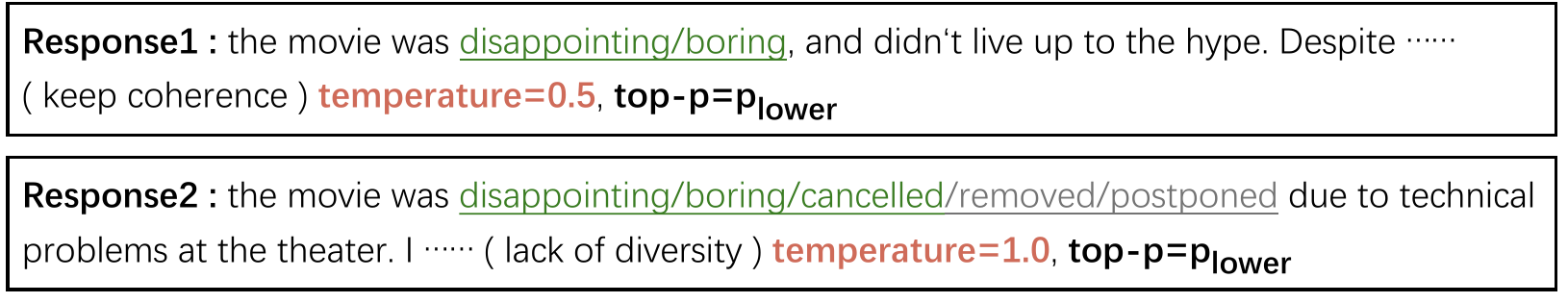}

\caption{Hyperparameter sensitivity across temperatures. EES achieves consistent optimal performance without tuning, while top-p requires temperature-specific hyperparameter adjustment.}
\label{fig:temperature_robustness}
\end{figure}

However, the reliance on auxiliary hyperparameters presents significant challenges for practical deployment~\cite{liao2022empirical}. While existing approaches can achieve optimal performance through careful hyperparameter tuning at specific temperature values, practical applications require methods that maintain consistent quality across varying temperature conditions~\cite{ackley1985learning} without the need for condition-specific parameter adjustments~\cite{yang2022tensor,liu2024large}. 



To address these fundamental limitations, we introduce \textit{Entropy Equilibrium Sampling (EES)}, a novel, streamlined, auxiliary hyperparameter-free approach grounded in information theory~\cite{cover1999elements}, which maintains high generation quality while simplifying deployment. Our method leverages the insight that optimal sampling can be achieved by balancing normalized entropy and probability mass within dynamically constructed candidate sets.

Figure~\ref{fig:temperature_robustness} demonstrates the critical challenge of maintaining optimal performance when temperature conditions change. The result reveals that methods like top-p achieve optimal token selection at specific temperatures (e.g., $\tau=0.5$) through careful hyperparameter tuning but fail to maintain this performance when temperature changes (e.g., $\tau=1.0$) without re-tuning their auxiliary parameters. In contrast, our proposed EES method consistently achieves optimal token truncation across varying temperature conditions without requiring any hyperparameter adjustment. This sensitivity not only complicates practical deployment but also makes it difficult to achieve consistent performance across different domains and applications~\cite{loya2023exploring,liang2022holistic}.

In summary, the contributions of this paper are as follows:
\begin{itemize}
\item \textbf{Theoretical analysis for entropy-probability equilibrium}: We provide rigorous theoretical analysis proving the existence and uniqueness of the equilibrium between normalized entropy and probability mass, establishing a solid mathematical foundation for auxiliary-hyperparameter-free sampling methods.
\item \textbf{Auxiliary-hyperparameter-free sampling method}: We present a method that dynamically adjusts candidate set size based on entropy-probability mass relationships, maintaining effective coherence-diversity balance without hyperparameter tuning.
\item \textbf{Comprehensive empirical validation}: We conducted extensive experiments across diverse benchmarks, model families, and scales, demonstrating that our method effectively balances coherence and diversity while maintaining robustness across different temperature settings.
\end{itemize}

%% file: sections/related_work.tex
\section{Related Work}
\subsection{Deterministic Methods}
Deterministic sampling methods provide reproducible outputs and strong coherence, but suffer from limited diversity due to their conservative selection strategies. Greedy sampling selects the highest probability token at each step~\cite{holtzman2019curious}. Beam search~\cite{freitag2017beam} maintains multiple candidate sequences but still tends to generate bland text due to its preference for high-probability paths~\cite{vijayakumar2016diverse}. Contrastive Search~\cite{su2022contrastive} balances model confidence with lexical diversity penalties, while Frustratingly Simple sampling~\cite{yang2023frustratingly} uses auxiliary anti-language models to guide generation. DoLa~\cite{chuang2023dola} leverages layer-wise representational differences in transformers for improved sampling.

\subsection{Stochastic Methods}

Stochastic sampling methods aim to identify an appropriate truncation point to construct a candidate set from the original vocabulary, followed by probabilistic sampling within this subset. The widely adopted top-k~\cite{fan2018hierarchical} and top-p~\cite{holtzman2019curious} sampling methods employ fixed truncation criteria based on candidate set size and probability mass, respectively. Recent methods introduce adaptive sampling strategies that dynamically adjust candidate sets based on distribution characteristics. $\eta$-sampling~\cite{hewitt2022truncation} uses entropy thresholds to control candidate set size based on prediction uncertainty. Mirostat sampling~\cite{basu2020mirostat} maintains text perplexity within a predetermined range by adaptively selecting tokens based on the current generation's perplexity. Typical sampling~\cite{meister2023locally} ranks vocabulary tokens by the divergence between their individual probabilities and the distribution's entropy, then applies a coverage threshold to determine the candidate set. Adaptive sampling~\cite{zhu2024improving} uses entropy-based confidence metrics to determine token inclusion in the candidate set during generation. Min-$p$~\cite{nguyen2024turning} dynamically adjusts sampling thresholds by utilizing the top token's probability as a scaling factor.

%% file: sections/preliminaries.tex
\section{Preliminaries}

\subsection{Autoregressive Sampling}

In autoregressive text generation, given a context sequence $x_{1:t-1}$, a LLM computes a score vector $\mathbf{s} \in \mathbb{R}^{|\mathcal{V}|}$ where each element $s_i$ represents the unnormalized logit for the $i$-th vocabulary token in vocabulary $\mathcal{V}$~\cite{vaswani2017attention,radford2019language}. This score vector undergoes temperature scaling with hyperparameter $\tau > 0$, followed by softmax normalization to obtain the probability distribution:
\begin{equation}
P(x_t = v_i | x_{1:t-1}) = \frac{\exp(s_i/\tau)}{\sum_{j=1}^{|\mathcal{V}|} \exp(s_j/\tau)},
\label{eq:softmax}
\end{equation}
where different sampling methods then employ various strategies to select tokens from this distribution, balancing between generation coherence and diversity~\cite{fan2018hierarchical,holtzman2019curious}.

\subsection{Normalized Entropy}
Normalized entropy measures the degree of uncertainty in probability distributions relative to the maximum possible uncertainty. Given the top-$k$ tokens sorted by descending probability, the normalized entropy is defined as:
\begin{equation}
\bar{H}_k := \frac{H_k}{\log k},
\label{eq:normalized_entropy}
\end{equation}
where $H_k$ is the Shannon entropy~\cite{shannon1948mathematical} of the top-$k$ subset:
\begin{equation}
H_k = -\sum_{i=1}^{k} \hat{p}_i \log \hat{p}_i,
\label{eq:shannon_entropy}
\end{equation}
and $\hat{p}_i = \frac{p_i}{\sum_{j=1}^{k}p_j}$ represents the renormalized probability of the $i$-th token within the top-$k$ subset.

The denominator $\log k$ represents the maximum entropy achievable when all $k$ tokens have equal probability ($\hat{p}_i = \frac{1}{k}$ for all $i$). Thus, $\bar{H}_k \in [0, 1]$, where:

\begin{itemize}
    \item $\bar{H}_k \to 0$: The distribution is highly concentrated, indicating low uncertainty and high confidence in token selection.
    \item $\bar{H}_k \to 1$: The distribution approaches uniform, indicating high uncertainty and approximately equal likelihood among tokens.
\end{itemize}

This normalized entropy measure provides a scale-invariant metric for assessing distributional uncertainty within candidate token sets, enabling consistent comparison across different vocabulary subset sizes.

\subsection{Probability Mass}
Probability mass measures the cumulative probability concentration within a token subset. Given the top-$k$ tokens ranked by descending probability, probability mass is defined as:
\begin{equation}
P_k := \sum_{i=1}^{k} p_i,
\label{eq:probability_mass}
\end{equation}
in contrast to entropy, which captures the shape of the distribution, probability mass directly reflects candidate subset coverage, making it suitable for truncation thresholds.

%% file: sections/method.tex
\section{Entropy Equilibrium Sampling}

\subsection{Problem Formulation}
In LLM sampling, the fundamental challenge lies in balancing exploration and exploitation within the vocabulary probability distribution~\cite{meister2021language,brown2024large}. This balance is characterized by two key information-theoretic measures: Normalized Entropy, which quantifies the uniformity of the probability distribution within the candidate set, and Probability Mass, which represents the cumulative probability coverage of selected candidates.

We formalize this as an optimization problem over vocabulary subsets. Given a conditional probability distribution $P(x_t | x_{<t})$ at generation step $t$, we seek a candidate set $A^* \subseteq \mathcal{V}$ that optimizes the balance between these two measures:
\begin{equation}
A^*_{x_{<t}} = \arg\max_{A^*_{x_{<t}} \subseteq \mathcal{V}} \mathcal{F}(\bar{H}_k(A^*_{x_{<t}}), \text{P}(A^*_{x_{<t}})),
\label{eq:optimization}
\end{equation}
where $\mathcal{F}$ represents the objective function balancing coherence-diversity trade-offs. We propose that an effective balance can be achieved when normalized entropy and probability mass are in equilibrium, suggesting that $\mathcal{F}$ should consider candidate sets where these measures are aligned.

\subsection{Algorithm Overview}
To implement the equilibrium-based approach outlined in Section 4.1, we instantiate the objective function $\mathcal{F}$ by seeking candidate sets where normalized entropy $\bar{H}_k$ equals the probability mass $P_k$. This alignment between the two measures provides a concrete criterion for identifying effective candidate sets.

Our proposed sampling method identifies the truncation point where this balance is achieved, thereby determining the candidate set. Since candidate set size selection is inherently discrete, we formulate this as finding the optimal threshold $k^*$:
\begin{equation}
k^* = \arg\max_{k}\left\{k: \bar{H}_k \geq P_k\right\}.
\label{eq:threshold}
\end{equation}
Once $k^*$ is determined, we truncate the vocabulary probability distribution at this threshold, retaining only the top-$k^*$ tokens as our candidate set for subsequent sampling.

\subsection{Convergence Analysis}

We establish the theoretical foundation for our algorithm by proving the existence and uniqueness of the optimal threshold $k^*$.

\begin{theorem}
For any probability distribution $\{p_i\}_{i=1}^n$ sorted in descending order ($p_1 \geq p_2 \geq \cdots \geq p_n > 0$), there exists a unique $k^* \in \{1, 2, \ldots, n\}$ such that the algorithm converges.
\end{theorem}

\begin{proof}
We define the objective function:
\begin{equation}
f(k) = \bar{H}_k - P_k,
\label{eq:objective_function}
\end{equation}
where the algorithm seeks the largest $k^*$ such that $f(k^*) \geq 0$. The proof consists of three steps.

\textbf{Step 1: Monotonicity of the objective function.}

We establish that $f(k) = \bar{H}_k - P_k$ is strictly decreasing. Let $g(k) = \bar{H}_k$ denote the normalized entropy. We demonstrate that $g(k+1) \leq g(k)$ for all valid $k$. For the boundary case $k=1$, we set $g(1) = 1$ and observe that $g(1) = 1 \geq g(2) = \frac{H_2}{\log 2}$ since $H_2 \leq \log 2$.

For $k \geq 2$, we prove:
\begin{equation}
\frac{H_{k+1}}{\log(k+1)} \leq \frac{H_k}{\log k}.
\label{eq:entropy_inequality}
\end{equation}

This is equivalent to:
\begin{equation}
H_{k+1} \cdot \log k \leq H_k \cdot \log(k+1).
\label{eq:main}
\end{equation}

We decompose $H_{k+1}$ as:
\begin{align}
H_{k+1} &= -\sum_{i=1}^{k} \frac{p_i}{P_{k+1}} \log \frac{p_i}{P_{k+1}} - \frac{p_{k+1}}{P_{k+1}} \log \frac{p_{k+1}}{P_{k+1}} \nonumber \\
&= \frac{P_k}{P_{k+1}} H_k + \frac{P_k}{P_{k+1}} \log \frac{P_{k+1}}{P_k} \nonumber \\
&\quad - \left(1-\frac{P_k}{P_{k+1}}\right) \log \left(1-\frac{P_k}{P_{k+1}}\right).
\label{eq:entropy_decomp}
\end{align}

Defining $\alpha = \frac{P_k}{P_{k+1}}$, we obtain:
\begin{equation}
H_{k+1} = \alpha H_k + h(\alpha),
\label{eq:decomp}
\end{equation}
where $h(\alpha) = -\alpha \log \alpha - (1-\alpha) \log(1-\alpha)$ is the binary entropy function.

Since $\{p_i\}_{i=1}^n$ is sorted in descending order, we have $\alpha = \frac{P_k}{P_{k+1}} \geq \frac{k}{k+1}$. 

When $\alpha = \frac{k}{k+1}$ (uniform case), we have $p_1 = p_2 = \cdots = p_{k+1}$, yielding $H_k = \log k$ and $H_{k+1} = \log(k+1)$, so inequality~\eqref{eq:main} holds with equality.

When $\alpha > \frac{k}{k+1}$, substituting eq.~\eqref{eq:decomp} into eq.~\eqref{eq:main} and rearranging:
\begin{equation}
H_k \cdot \left( \log(k+1) - \alpha \log k \right) \geq \log k \cdot h(\alpha).
\label{eq:rearranged}
\end{equation}

Define $\phi(\alpha) = H_k \cdot(\log(k+1) - \alpha \log k) - \log k \cdot h(\alpha)$. Taking the derivative:
\begin{equation}
\phi'(\alpha) = -\log k \cdot \left(H_k + \log\frac{1-\alpha}{\alpha}\right).
\label{eq:derivative}
\end{equation}

For $\alpha > \frac{k}{k+1}$, we have $\frac{1-\alpha}{\alpha} < \frac{1}{k}$, implying $\log\left(\frac{1-\alpha}{\alpha}\right) < -\log k$. Therefore:
\begin{equation}
\phi'(\alpha) > -\log k \cdot (H_k - \log k) \geq 0,
\label{eq:derivative_bound}
\end{equation}
where the inequality follows from $H_k \leq \log k$.

Since $\phi\left(\frac{k}{k+1}\right) = 0$ and $\phi'(\alpha) > 0$ for $\alpha \in \left(\frac{k}{k+1}, 1\right)$, we have $\phi(\alpha) > 0$ in this interval, establishing inequality~\eqref{eq:rearranged}.

Therefore, $g(k)$ is strictly decreasing. Since $P_k$ is strictly increasing, $f(k) = g(k) - P_k$ is strictly decreasing.

\textbf{Step 2: Boundary behavior analysis.}

We examine the behavior of $f(k)$ at the boundaries to establish the existence of a zero crossing.

For $k=1$: Since we define $g(1) = 1$ and $P_1 = p_1 \leq 1$, we have $f(1) = 1 - p_1 \geq 0$.

As $k$ approaches the vocabulary size $n$: We have $P_n = 1$ and $\bar{H}_n \leq 1$, with strict inequality for non-uniform distributions due to entropy concavity. Thus:
\begin{equation}
\lim_{k \to n} f(k) = \bar{H}_n - 1 \leq 0.
\label{eq:boundary_limit}
\end{equation}

\textbf{Step 3: Existence and uniqueness.}

Since $f(k)$ is strictly decreasing on the discrete domain $\{1, 2, \ldots, n\}$, with $f(1) \geq 0$ and $f(n) \leq 0$, there exists a unique optimal threshold $k^*$ defined as:
\begin{equation}
k^* = \arg\max_k \{k : f(k) \geq 0\}.
\label{eq:unique_threshold}
\end{equation}

The strict monotonicity of $f(k)$ ensures that:
\begin{itemize}
    \item If $k^* < n$, then $f(k^*) \geq 0$ and $f(k^* + 1) < 0$
    \item If $k^* = n$, then $f(n) \geq 0$ (uniform distribution case)
\end{itemize}

This guarantees the uniqueness of the optimal threshold and ensures our algorithm converges to this well-defined solution.
\end{proof}

\subsection{Algorithm Implementation}

Figure~\ref{fig:hyperparameter_analysis} illustrates the mechanism of EES, showing how the algorithm determines the largest $k$ value as the optimal threshold $k^*$ where normalized entropy still exceeds or equals probability mass. Algorithm~\ref{alg:ees} presents the complete algorithmic procedure for EES, which dynamically selects the optimal sampling threshold at each generation step by maintaining a normalized entropy-probability mass equilibrium condition. The algorithm incrementally expands the candidate set while monitoring the condition $\bar{H}_k \geq P_k$, terminating when this equilibrium is first violated to identify the optimal threshold $k^*$. This approach automatically adapts to varying model confidence levels without requiring auxiliary hyperparameters.

\begin{figure}[H]
\centering
\includegraphics[width=0.40\textwidth]{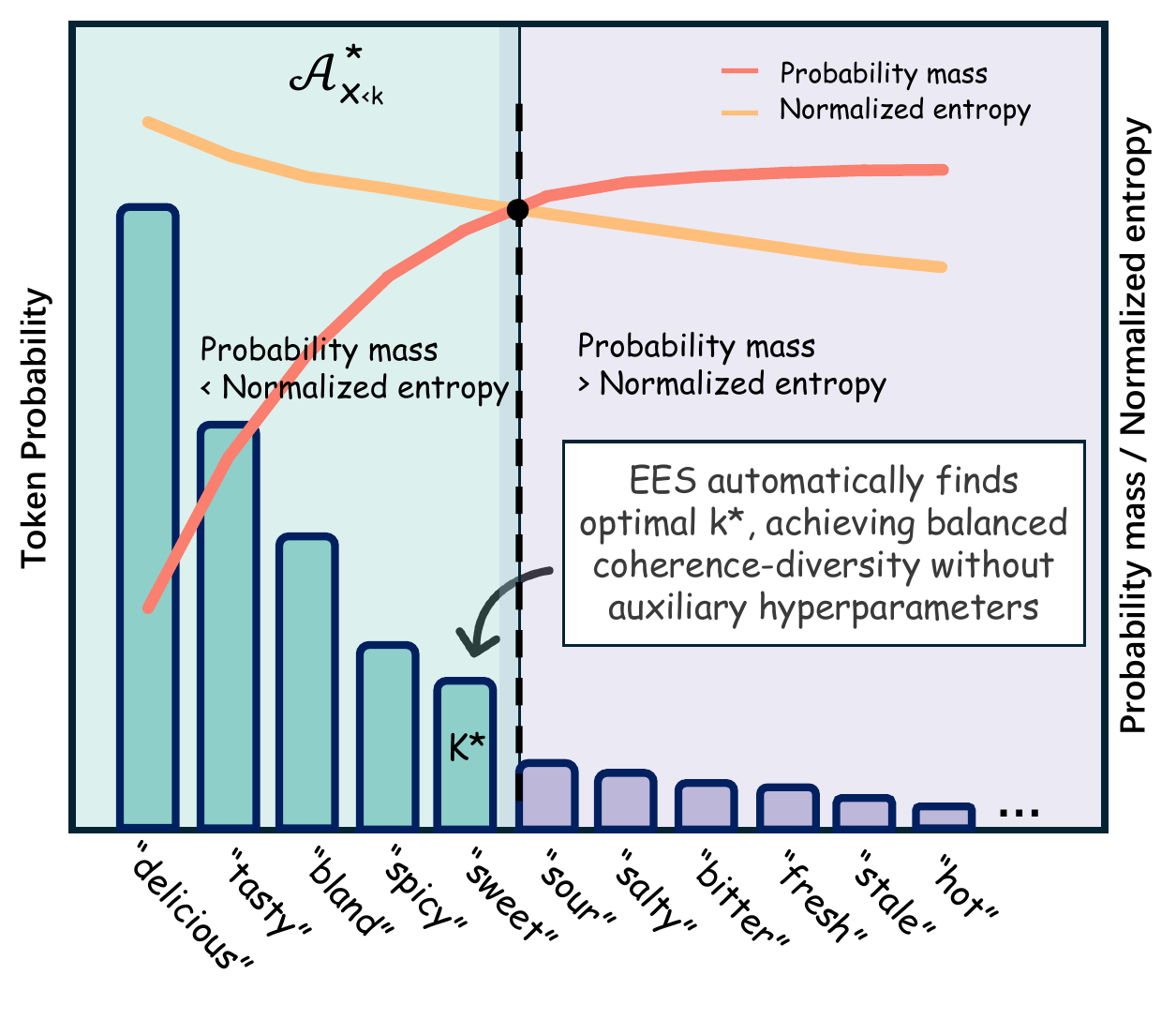}
\vspace{-0.5cm} 
\caption{Mechanism of EES}
\label{fig:hyperparameter_analysis}
\end{figure}

\textbf{Complexity Analysis.} The computational complexity per generation step consists of: (1) probability computation: $\mathcal{O}(|\mathcal{V}|)$, (2) sorting: $\mathcal{O}(|\mathcal{V}|\log|\mathcal{V}|)$, (3) threshold selection: $\mathcal{O}(k^2)$ for entropy computation across iterations, and (4) sampling: $\mathcal{O}(k)$. Since $k \ll |\mathcal{V}|$ in practice, the overall complexity is $\mathcal{O}(|\mathcal{V}|\log|\mathcal{V}|)$ per generation step.

\begin{algorithm}
\caption{Entropy Equilibrium Sampling (EES)}
\label{alg:ees}
\begin{algorithmic}[1]
\REQUIRE Prefix $\mathcal{C} = \{x_1, \ldots, x_m\}$; Language Model $P_\theta$; Vocabulary size $|\mathcal{V}|$
\ENSURE Continuation $\mathcal{G} = \{x_{m+1}, x_{m+2}, \ldots, x_{m+n}\}$

\STATE $\mathcal{G} \leftarrow \emptyset$
\FOR{each generation step}
    \STATE $\mathbf{p} \leftarrow \text{sort}(P_\theta(x_t | \mathcal{C}), \text{descending})$
    \STATE $\mathcal{A} \leftarrow \emptyset$, $P_0 \leftarrow 0$ 
    
    \FOR{$k = 1$ to $|\mathcal{V}|$}
        \STATE $P_k \leftarrow P_{k-1} + p_k$, $\mathcal{A} \leftarrow \mathcal{A} \cup \{p_k\}$
        \IF{$k \geq 2$}
            \STATE $\mathbf{q} \leftarrow \{p_i / P_k : p_i \in \mathcal{A}\}$ 
            \STATE $\bar{H}_k \leftarrow \frac{-\sum_{q \in \mathbf{q}} q \log q}{\log k}$
            \IF{$\bar{H}_k < P_k$}
                \STATE $k^* \leftarrow k-1$, \textbf{break}
            \ENDIF
        \ENDIF
    \ENDFOR
    \STATE $x_t \leftarrow \text{sample from top-}k^* \text{ distribution}$
    \STATE Update $\mathcal{C}$ and $\mathcal{G}$ with $x_t$
\ENDFOR

\RETURN $\mathcal{G}$
\end{algorithmic}
\end{algorithm}

%% file: sections/exp.tex
\section{Experiments}
\subsection{Setups}

\begin{table*}[tb]
\centering
\renewcommand{\arraystretch}{0.9} 
\resizebox{\textwidth}{!}{
\scriptsize
\begin{tabular}{cc@{\hspace{0.6em}}l@{\hspace{0.6em}}l@{\hspace{0.6em}}l@{\hspace{1.0em}}l@{\hspace{1.8em}}l@{\hspace{0.6em}}l@{\hspace{0.6em}}l@{\hspace{1.0em}}l}
\toprule
\multirow{3}{*}{\small\textbf{LM}} & \multirow{3}{*}{\small\textbf{Sampling Method}} & \multicolumn{4}{c}{\footnotesize\textbf{CommonsenseQA}} & \multicolumn{4}{c}{\footnotesize\textbf{StrategyQA}} \\
\cmidrule(lr){3-6} \cmidrule(lr){7-10}
& & $\tau=0.5$ & $\tau=0.8$ & $\tau=1.0$ & \textbf{Avg} & $\tau=0.5$ & $\tau=0.8$ & $\tau=1.0$ & \textbf{Avg} \\
\midrule
\multirow{8}{*}{\rotatebox{90}{\small\textbf{Qwen2.5-7B}}} 
& Temperature & 83.70 & 82.39 & 82.01 & 82.70 & 75.75 & 75.90 & 74.67 & 75.44 \\
& Top-p & 83.77$_{(0.75)}$ & 82.64$_{(0.80)}$ & 81.93$_{(0.80)}$ & 82.78 & 74.79$_{(0.90)}$ & 75.87$_{(0.85)}$ & 75.05$_{(0.80)}$ & 75.24 \\
& Top-k & 82.75$_{(10)}$ & 82.42$_{(100)}$ & 82.65$_{(20)}$ & 82.61 & \underline{75.87}$_{(5)}$ & 75.11$_{(5)}$ & 74.44$_{(100)}$ & 75.14 \\
& Eta & 82.98$_{(0.0009)}$ & 81.88$_{(0.002)}$ & 82.10$_{(0.0006)}$ & 82.32 & 75.43$_{(0.0006)}$ & 74.99$_{(0.0006)}$ & 75.11$_{(0.002)}$ & 75.18 \\
& Mirostat & \underline{83.88}$_{(2.5)}$ & 83.23$_{(2.5)}$ & \underline{83.55}$_{(2.5)}$ & \underline{83.55} & 75.72$_{(2.5)}$ & \underline{76.33}$_{(3.5)}$ & 75.37$_{(2.5)}$ & \underline{75.81} \\
& Typical & 83.36$_{(0.95)}$ & \textbf{84.11}$_{(0.2)}$ & 82.95$_{(0.2)}$ & 83.47 & 75.34$_{(0.2)}$ & 75.57$_{(0.2)}$ & 75.23$_{(0.2)}$ & 75.38 \\
& Adaptive & 83.62$_{(0.0005)}$ & 83.46$_{(0.005)}$ & 83.39$_{(0.01)}$ & 83.49 & 75.63$_{(0.005)}$ & 75.52$_{(0.005)}$ & \underline{75.46}$_{(0.005)}$ & 75.54 \\
& \textbf{Ours} & \textbf{84.42} & \underline{83.95} & \textbf{83.64} & \textbf{84.00} & \textbf{76.33} & \textbf{76.56} & \textbf{75.95} & \textbf{76.28} \\
\midrule
\multirow{8}{*}{\rotatebox{90}{\small\textbf{Llama3.1-8B}}} 
& Temperature & 76.69 & 75.71 & 73.63 & 75.34 & 73.80 & 72.34 & 69.02 & 71.72 \\
& Top-p & 76.81$_{(0.75)}$ & 76.25$_{(0.75)}$ & 75.28$_{(0.75)}$ & 76.11 & 73.54$_{(0.95)}$ & 73.54$_{(0.85)}$ & 72.66$_{(0.75)}$ & 73.25 \\
& Top-k & \underline{77.33}$_{(20)}$ & 76.23$_{(5)}$ & 74.76$_{(5)}$ & 76.11 & \textbf{75.28}$_{(5)}$ & 72.20$_{(20)}$ & 72.52$_{(5)}$ & 73.33 \\
& Eta & 77.17$_{(0.0006)}$ & 75.22$_{(0.0003)}$ & 74.25$_{(0.004)}$ & 75.55 & 74.06$_{(0.002)}$ & 71.99$_{(0.0003)}$ & 70.39$_{(0.0009)}$ & 72.15 \\
& Mirostat & 77.20$_{(2.5)}$ & 76.38$_{(3.0)}$ & 76.54$_{(2.5)}$ & 76.71 & 74.12$_{(4.0)}$ & \textbf{75.46}$_{(3.0)}$ & 74.50$_{(4.0)}$ & 74.69 \\
& Typical & 77.22$_{(0.2)}$ & 75.81$_{(0.2)}$ & 75.18$_{(0.2)}$ & 76.07 & 75.05$_{(0.2)}$ & 73.57$_{(0.2)}$ & 72.02$_{(0.90)}$ & 73.55 \\
& Adaptive & 77.17$_{(0.01)}$ & \textbf{77.38}$_{(0.0005)}$ & \textbf{77.20}$_{(0.01)}$ & \textbf{77.25} & 75.11$_{(0.0005)}$ & \underline{75.17}$_{(0.01)}$ & \textbf{75.14}$_{(0.0005)}$ & \textbf{75.14} \\
& \textbf{Ours} & \textbf{77.44} & \underline{76.72} & \underline{76.90} & \underline{77.02} & \underline{75.14} & 74.50 & \underline{74.61} & \underline{74.75} \\
\midrule
\multirow{8}{*}{\rotatebox{90}{\small\textbf{Qwen2.5-32B}}} 
& Temperature & 88.98 & 88.27 & 88.52 & 88.59 & 80.29 & 79.39 & 79.04 & 79.57 \\
& Top-p & 88.63$_{(0.90)}$ & 88.21$_{(0.95)}$ & 88.19$_{(0.85)}$ & 88.34 & \textbf{80.44}$_{(0.80)}$ & 79.42$_{(0.95)}$ & 79.24$_{(0.80)}$ & 79.70 \\
& Top-k & 88.83$_{(10)}$ & \textbf{89.12}$_{(5)}$ & 88.11$_{(50)}$ & 88.69 & 79.77$_{(5)}$ & 80.38$_{(5)}$ & 78.57$_{(100)}$ & 79.57 \\
& Eta & 88.99$_{(0.002)}$ & 88.48$_{(0.002)}$ & 88.47$_{(0.002)}$ & 88.65 & 80.15$_{(0.0009)}$ & 78.95$_{(0.004)}$ & 79.39$_{(0.004)}$ & 79.50 \\
& Mirostat & \underline{89.24}$_{(3.0)}$ & 88.75$_{(4.0)}$ & \textbf{89.09}$_{(2.5)}$ & \textbf{89.03} & 80.32$_{(3.5)}$ & 80.12$_{(2.5)}$ & 80.15$_{(3.5)}$ & 80.20 \\
& Typical & 89.19$_{(0.2)}$ & \underline{88.85}$_{(0.2)}$ & 88.70$_{(0.92)}$ & 88.91 & 80.06$_{(0.90)}$ & 80.32$_{(0.92)}$ & 79.45$_{(0.92)}$ & 79.94 \\
& Adaptive & 88.98$_{(0.0005)}$ & 88.76$_{(0.01)}$ & 88.83$_{(0.001)}$ & 88.86 & 80.32$_{(0.0005)}$ & \textbf{80.79}$_{(0.0005)}$ & \textbf{80.99}$_{(0.01)}$ & \textbf{80.70} \\
& \textbf{Ours} & \textbf{89.30} & \underline{88.85} & \underline{88.94} & \textbf{89.03} & \underline{80.35} & \underline{80.52} & \underline{80.23} & \underline{80.37} \\
\bottomrule
\end{tabular}
}
\caption{Cross-model performance comparison of sampling methods under different temperature settings on two QA datasets. Bold: best performance; underlined: second-best; parentheses: optimal auxiliary hyperparameter.}
\label{tab:sampling_comparison_acc}
\end{table*}


\paragraph{Models and Datasets} We evaluated our method using Qwen2.5-7B, Qwen2.5-32B~\cite{yang2025qwen3}, and Llama3.1-8B~\cite{dubey2024llama} on CommonsenseQA~\cite{talmor2018commonsenseqa} and StrategyQA~\cite{geva2021did}. We used the proportion of correctly answered questions as our evaluation metric.

\paragraph{Baseline} We compared against seven established sampling methods: temperature sampling, top-p, top-k sampling, eta sampling, mirostat, typical sampling, and adaptive sampling.

\subsection{Experimental Protocol}
We randomly sampled 20\% of the training data as a validation set to identify optimal hyperparameters for each sampling method across different temperature settings. We configured the hyperparameter search space for each method based on recommendations from the original papers and subsequent literature~\cite{shi2024thorough}:

\begin{itemize}
    \item \textbf{Nucleus (top-p)}: $p \in \{0.75, 0.8, 0.85, 0.9, 0.95\}$
    \item \textbf{Top-k}: $k \in \{5, 10, 20, 50, 100\}$
    \item \textbf{Eta sampling}: $\eta \in \{3 \times 10^{-4}, 6 \times 10^{-4}, 9 \times 10^{-4}, 2 \times 10^{-3}, 4 \times 10^{-3}\}$
    \item \textbf{Mirostat}: $\tau \in \{2.5, 3.0,3.5, 4.0\}$
    \item \textbf{Typical sampling}: $p \in \{0.2, 0.9, 0.92, 0.95\}$
    \item \textbf{Adaptive sampling}: $\epsilon \in \{5 \times 10^{-4}, 1 \times 10^{-3}, 5 \times 10^{-3}, 1 \times 10^{-2}\}$
\end{itemize}

For each configuration, we conducted three independent runs and selected the hyperparameter setting that achieves the highest average accuracy on the validation set. On the test set, we evaluated each method using its optimal hyperparameter with five independent runs, reporting average performance as the final score. Following~\cite{shi2024thorough}, we constructed model inputs by randomly sampling multiple examples from the training set to provide in-context learning examples.

\subsection{Results}

\begin{figure*}[t]
    \centering
    \begin{subfigure}{0.48\textwidth}
        \centering
        \includegraphics[width=\textwidth]{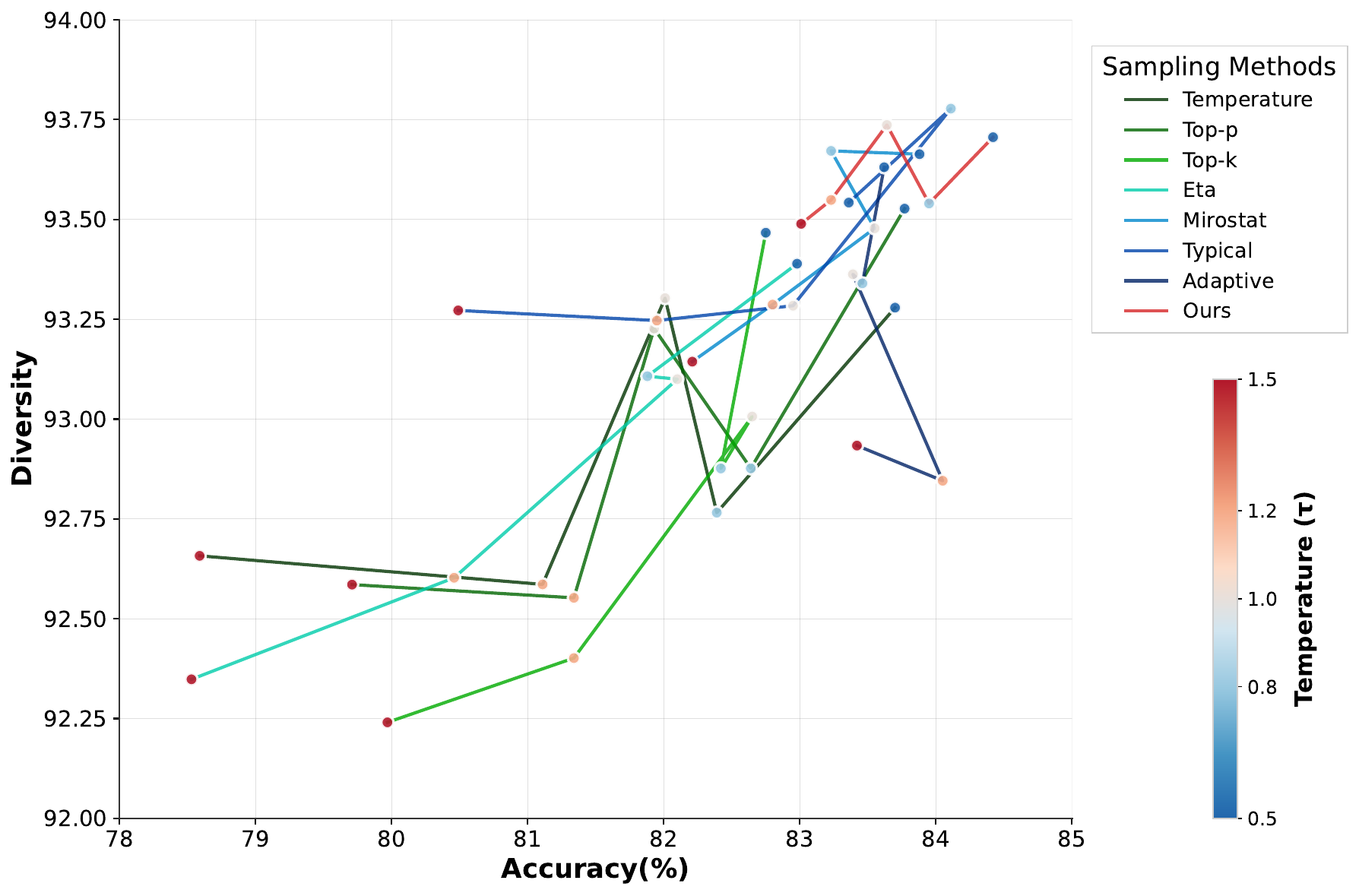}
        \caption{CommonsenseQA}
        \label{fig:commonsenseqa_results}
    \end{subfigure}
    \hfill
    \begin{subfigure}{0.48\textwidth}
        \centering
        \includegraphics[width=\textwidth]{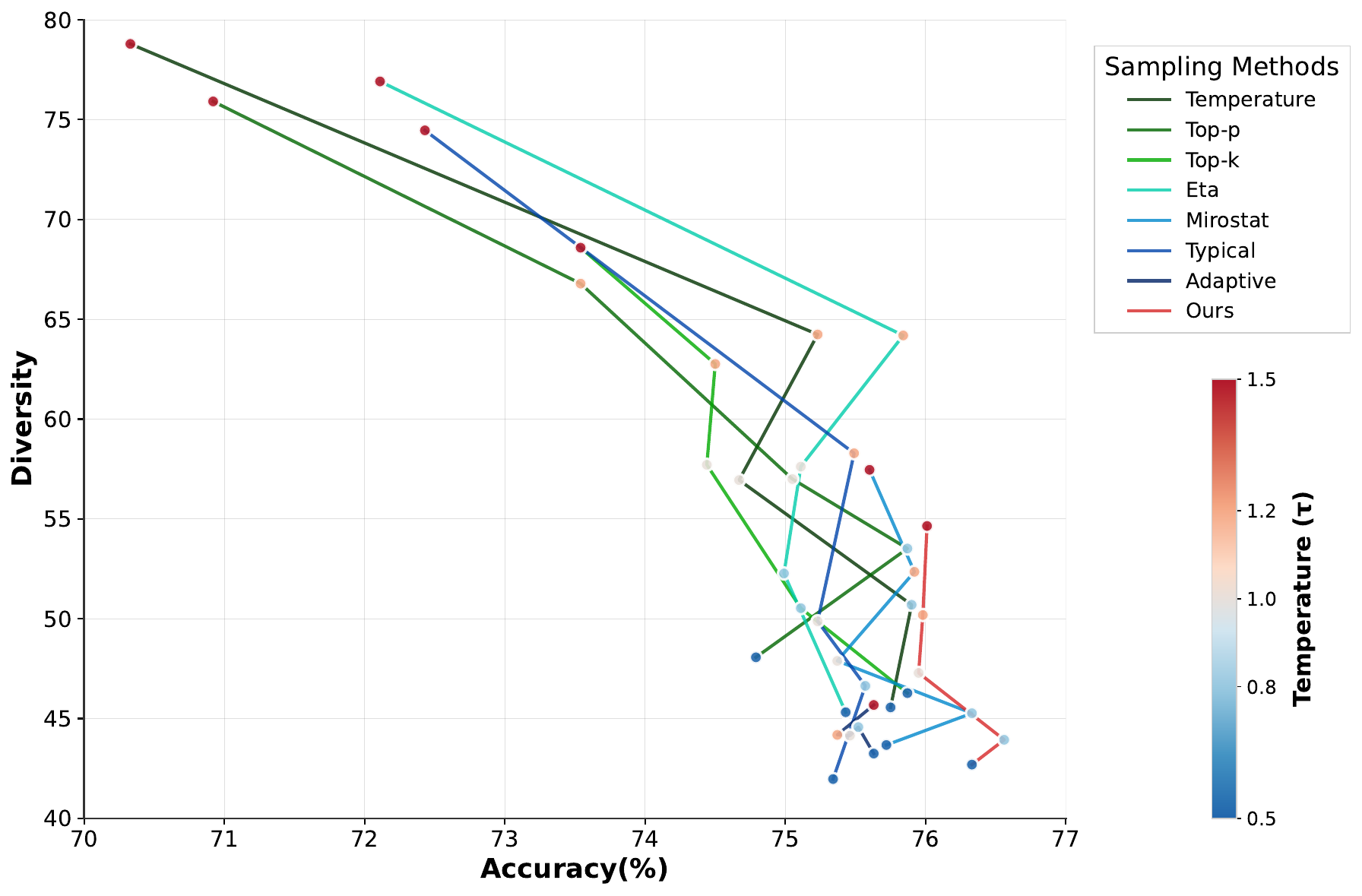}
        \caption{StrategyQA}
        \label{fig:strategyqa_results}
    \end{subfigure}
    \caption{Accuracy-diversity performance across sampling methods and temperatures on two QA datasets using Llama3.1-8B.}
    \label{fig:accuracy_diversity_tradeoff}
\end{figure*}

Table~\ref{tab:sampling_comparison_acc} reveals that our EES method consistently achieves competitive performance, ranking among the top methods across all experimental configurations while demonstrating robust performance across varying temperature settings and model architectures. While some baseline methods occasionally achieve slightly higher individual scores (e.g., Adaptive sampling on Llama3.1-8B), our approach offers a distinct advantage in practical deployment scenarios.

Most significantly, our method eliminates auxiliary hyperparameter sensitivity: baseline methods require diverse optimal configurations across different models and temperatures, with parameters varying dramatically between settings (e.g., top-k from 5 to 100, Eta from 0.0003 to 0.004), whereas our approach maintains consistent performance without any hyperparameter adjustment. Notably, we observed a clear scaling effect where performance variance between methods decreases substantially as model size increases—accuracy differences compress from 2-3\% gaps in smaller models to $<$1\% in the 32B parameter regime, suggesting improved calibration in larger models.

\subsection{Diversity analysis}
\label{subsec:diversity_analysis}

We evaluated diversity using the repetition-based metric $Diversity = \prod_{n=2}^{4}(1.0 - \frac{\text{rep-}n}{100})$ from~\cite{meister2023locally} on correctly classified responses from Qwen2.5-7B. Figure~\ref{fig:accuracy_diversity_tradeoff} demonstrates that our method achieves competitive diversity while maintaining superior accuracy across all temperature settings. Notably, we observed that the presumed accuracy-diversity trade-off is highly task-dependent: on CommonsenseQA (Figure~\ref{fig:commonsenseqa_results}), all sampling methods reach high diversity scores ($>$93\%) at $\tau=0.5$, with higher temperatures providing negligible diversity gains while significantly degrading accuracy. This suggests that task complexity fundamentally modulates the optimal operating point for sampling strategies, and that the traditional assumption of a universal accuracy-diversity trade-off may not hold across different reasoning tasks.

\subsection{Creative Writing}

We evaluated open-ended text generation performance on WikiText-103~\cite{merity2016pointer} using Qwen2.5-7B and Llama3.1-8B models at elevated temperature settings $\tau \in \{1.0, 1.2, 1.5\}$ to promote diversity. We constructed a validation set of 200 randomly sampled training instances for hyperparameter optimization via MAUVE maximization~\cite{pillutla2021mauve}. Each generation task utilizes 32-word prefixes as prompts from text samples containing at least 150 words, with model outputs limited to 256 tokens maximum. We employed three evaluation metrics: (1) \textbf{MAUVE}~\cite{pillutla2021mauve} quantifying distributional similarity between generated and reference text, (2) \textbf{Diversity} measured as $\prod_{n=2}^{4}(1.0 - \frac{\text{rep-}n}{100})$, and (3) \textbf{Semantic Coherence}~\cite{gao2021simcse} computed as $\frac{\mathbf{v}_C^T \mathbf{v}_G}{|\mathbf{v}_C| \cdot |\mathbf{v}_G|}  $ where $\mathbf{v}_C = \text{SimCSE}(C)$ and $\mathbf{v}_G = \text{SimCSE}(G)$ represent SimCSE embeddings  of the generated and reference text respectively.


Table~\ref{tab:wikitext_comparison} presents our method achieving competitive MAUVE scores while maintaining consistent quality across all temperature ranges. Interestingly, we observe a task-dependent performance pattern: methods that performed poorly on QA datasets (e.g., top-k) demonstrate strong performance on WikiText-103, while some methods that excelled in QA tasks show degraded performance in text generation. This highlights the challenge of method selection across different tasks and the need for task-specific hyperparameter tuning in baseline approaches.

Our approach demonstrates superior temperature robustness: baseline methods show substantial performance drops at elevated temperatures (e.g., top-p drops from 29.08 to 1.48 MAUVE on Llama3.1-8B), while our method maintains consistent performance without hyperparameter adjustment. This stability, combined with consistent cross-task performance, makes our approach particularly suitable for deployment scenarios requiring reliable generation quality across varying temperature conditions and different task types, eliminating the need for both temperature-specific and task-specific parameter tuning that other methods require.

\begin{table*}[htbp]
\centering
\renewcommand{\arraystretch}{0.9} 
\resizebox{\textwidth}{!}{
\scriptsize
\begin{tabular}{ll@{\hspace{0.5em}}l@{\hspace{0.5em}}c@{\hspace{0.5em}}c@{\hspace{1.5em}}l@{\hspace{0.5em}}c@{\hspace{0.5em}}c@{\hspace{1.5em}}l@{\hspace{0.5em}}c@{\hspace{0.5em}}c@{\hspace{2.0em}}l@{\hspace{0.5em}}c@{\hspace{0.5em}}c}
\toprule
\multirow{3}{*}{\small\textbf{LM}} & \multirow{3}{*}{\small\textbf{Sampling Method}} & \multicolumn{12}{c}{\footnotesize\textbf{WikiText-103}} \\
\cmidrule(lr){3-14}
& & \multicolumn{3}{c}{$\tau=1.0$} & \multicolumn{3}{c}{$\tau=1.2$} & \multicolumn{3}{c}{$\tau=1.5$} & \multicolumn{3}{c}{\textbf{Avg}} \\
\cmidrule(lr){3-5} \cmidrule(lr){6-8} \cmidrule(lr){9-11} \cmidrule(lr){12-14}
& & MAUVE & Div & Coh & MAUVE & Div & Coh & MAUVE & Div & Coh & MAUVE & Div & Coh \\
\midrule
\multirow{8}{*}{\rotatebox{90}{\small\textbf{Qwen2.5-7B}}} 
& Temperature & \underline{34.13} & 59.07 & 80.04 & 31.48 & 62.80 & 79.41 & 21.81 & \textbf{82.53} & 69.44 & 29.14 & 68.13 & 76.30 \\
& Top-p & 25.97$_{(0.95)}$ & \textbf{68.40} & 80.74 & 19.34$_{(0.85)}$ & \textbf{70.65} & 79.95 & 22.44$_{(0.8)}$ & \underline{81.17} & 74.23 & 22.58 & \textbf{73.41} & 78.31 \\
& Top-k & \textbf{34.61}$_{(10)}$ & 63.89 & \textbf{81.05} & 29.43$_{(5)}$ & 64.62 & 80.90 & 29.42$_{(50)}$ & 72.30 & 80.72 & \underline{31.15} & 66.94 & 80.89 \\
& Eta & 21.67$_{(0.004)}$ & \underline{65.58} & 80.95 & \textbf{36.16}$_{(0.0003)}$ & \underline{69.63} & 80.77 & 23.09$_{(0.004)}$ & 79.98 & 76.72 & 26.97 & \underline{71.73} & 79.48 \\
& Mirostat & 24.40$_{(4.0)}$ & 61.38 & 80.94 & 28.93$_{(4.0)}$ & 63.49 & \textbf{81.02} & 22.95$_{(2.5)}$ & 64.32 & \textbf{81.04} & 25.43 & 63.06 & \textbf{80.97} \\
& Typical & 26.78$_{(0.95)}$ & 65.50 & \underline{80.97} & 22.95$_{(0.95)}$ & 69.22 & 80.49 & 19.88$_{(0.95)}$ & 79.31 & 75.24 & 23.20 & 71.34 & 78.90 \\
& Adaptive & 31.32$_{(0.0005)}$ & 56.43 & 80.79 & 29.75$_{(0.01)}$ & 55.48 & 80.73 & 25.87$_{(0.01)}$ & 54.29 & 80.76 & 28.98 & 55.40 & 80.76 \\
& \textbf{Ours} & 29.94 & 61.31 & 80.95 & \underline{33.67} & 63.46 & \underline{80.92} & \textbf{30.73} & 65.61 & \underline{81.02} & \textbf{31.45} & 63.46 & \underline{80.96} \\
\midrule
\multirow{8}{*}{\rotatebox{90}{\small\textbf{Llama3.1-8B}}} 
& Temperature & 25.41 & 51.26 & 77.32 & 2.87 & \textbf{97.03} & 57.53 & 1.42 & \textbf{99.81} & 51.92 & 9.90 & \textbf{82.70} & 62.26 \\
& Top-p & \textbf{29.08}$_{(0.8)}$ & 57.24 & 79.11 & 13.43$_{(0.75)}$ & 80.05 & 77.43 & 1.48$_{(0.9)}$ & \underline{99.67} & 57.18 & 14.66 & 78.99 & 71.24 \\
& Top-k & 24.82$_{(100)}$ & 59.87 & 79.11 & \textbf{22.88}$_{(10)}$ & 59.90 & \textbf{78.87} & \underline{27.70}$_{(20)}$ & 77.78 & 79.15 & \textbf{25.13} & 65.85 & \textbf{79.04} \\
& Eta & 21.20$_{(0.004)}$ & \textbf{62.71} & \textbf{79.36} & 12.68$_{(0.004)}$ & \underline{85.85} & 76.02 & 1.46$_{(0.0006)}$ & 99.62 & 57.71 & 11.78 & \underline{82.73} & 71.03 \\
& Mirostat & 14.54$_{(4.0)}$ & 36.06 & 78.26 & 14.44$_{(3.0)}$ & 38.29 & \underline{78.48} & 16.45$_{(3.0)}$ & 47.38 & 78.69 & 15.14 & 40.58 & \underline{78.48} \\
& Typical & \underline{24.94}$_{(0.2)}$ & \underline{61.15} & \textbf{79.36} & 5.46$_{(0.90)}$ & 89.39 & 71.16 & 1.23$_{(0.92)}$ & 99.66 & 56.29 & 10.54 & 83.40 & 68.94 \\
& Adaptive & 8.71$_{(0.001)}$ & 21.63 & 77.42 & 6.94$_{(0.005)}$ & 20.62 & 77.54 & 6.36$_{(0.0005)}$ & 20.50 & 77.34 & 7.34 & 20.92 & 77.43 \\
& \textbf{Ours} & 15.04 & 35.18 & 78.14 & \underline{14.70} & 40.16 & 78.41 & \textbf{29.43} & 51.52 & \underline{78.86} & \underline{19.72} & 42.29 & 78.47 \\
\bottomrule
\end{tabular}
}
\caption{Cross-model performance comparison of sampling methods under different temperature settings on WikiText-103. Bold: best performance; underlined: second-best; parentheses: optimal auxiliary hyperparameter.}
\label{tab:wikitext_comparison}
\end{table*}

\subsection{Human and LLM Evaluation}

To comprehensively assess generation quality beyond automatic metrics, we conducted both human and LLM-based evaluation on the generated outputs from Section 5.5. We collected generated texts from Llama3.1-8B at $\tau=1.0$ using optimal hyperparameters determined for each sampling method.

\paragraph{Evaluation Protocol.} We adopted a pairwise comparison approach framework following~\cite{zhu2024improving}. For each baseline method, we sampled 50 text pairs comparing our proposed method against each baseline using a 3-point Likert scale: +1 indicates clear superiority of our method, -1 indicates clear superiority of the baseline, and 0 indicates comparable quality. We excluded temperature sampling due to its limited practical adoption and included human-written reference texts from WikiText-103.

Following~\cite{nguyen2024turning}, we evaluated outputs across two dimensions:
\textbf{Quality}: Evaluating coherence, relevance to the given prompt, and overall text fluency. \textbf{Diversity}: Assessing creativity and distinctiveness of the generated content. For human evaluation, we recruited participants fluent in English with experience in text assessment tasks. For automated evaluation, we utilized Claude-3.7-Sonnet~\cite{anthropic2025claude37} with carefully crafted prompts, leveraging its demonstrated effectiveness in text evaluation tasks.

\begin{table}[htbp]
\centering
\resizebox{\columnwidth}{!}{%
\begin{tabular}{l|ccc|ccc}
\toprule
\multirow{2}{*}{\textbf{Comparison}} & \multicolumn{3}{c|}{\textbf{Quality}} & \multicolumn{3}{c}{\textbf{Diversity}} \\
\cmidrule(lr){2-4} \cmidrule(lr){5-7}
& \textbf{Better} & \textbf{Equal} & \textbf{Worse} & \textbf{Better} & \textbf{Equal} & \textbf{Worse} \\
\midrule
vs Human & 4 & 36 & 10 & 8 & 16 & 26 \\
vs Top-p & \textbf{11} & 29 & 10 & 11 & 15 & 24 \\
vs Top-k & \textbf{11} & 30 & 9 & 15 & 13 & 22 \\
vs Eta & \textbf{23} & 18 & 9 & \textbf{22} & 16 & 11 \\
vs Mirostat & \textbf{11} & 31 & 8 & \textbf{20} & 12 & 18 \\
vs Typical & \textbf{16} & 21 & 13 & 17 & 8 & 25 \\
vs Adaptive & \textbf{18} & 28 & 4 & \textbf{24} & 17 & 9 \\
\bottomrule
\end{tabular}%
}
\caption{Human evaluation: pairwise comparison results}
\label{tab:human_evaluation}
\end{table}

\textbf{Human Evaluation Results.} Table~\ref{tab:human_evaluation} shows the human evaluation results corroborating automated findings. Our method demonstrates competitive performance with positive win-loss ratios in quality evaluation for all baseline methods and shows strong diversity performance in most cases. When compared to human-written text, our method achieves reasonable performance with a notable number of ties (36 out of 50 in quality), indicating that our approach can produce text approaching human-level quality in certain contexts.

\begin{table}[htbp]
\centering
\resizebox{\columnwidth}{!}{%
\begin{tabular}{l|ccc|ccc}
\toprule
\multirow{2}{*}{\textbf{Comparison}} & \multicolumn{3}{c|}{\textbf{Quality}} & \multicolumn{3}{c}{\textbf{Diversity}} \\
\cmidrule(lr){2-4} \cmidrule(lr){5-7}
& \textbf{Better} & \textbf{Equal} & \textbf{Worse} & \textbf{Better} & \textbf{Equal} & \textbf{Worse} \\
\midrule
vs Human & 0 & 1 & 49 & 0 & 6 & 44 \\
vs Top-p & 17 & 9 & 24 & 8 & 14 & 28 \\
vs Top-k & 14 & 8 & 28 & 2 & 12 & 36 \\
vs Eta & \textbf{21} & 22 & 7 & \textbf{17} & 21 & 12 \\
vs Mirostat & \textbf{20} & 12 & 18 & 5 & 16 & 29 \\
vs Typical & \textbf{24} & 9 & 17 & 10 & 14 & 26 \\
vs Adaptive & \textbf{22} & 17 & 11 & \textbf{16} & 24 & 10 \\
\bottomrule
\end{tabular}%
}
\caption{LLM evaluation: pairwise comparison results using Claude-3.7-Sonnet}
\label{tab:llm_evaluation}
\end{table}

\textbf{LLM Evaluation Results.} Table~\ref{tab:llm_evaluation} presents the automated evaluation results using Claude-3.7-Sonnet. Our method achieves positive win-loss ratios in quality for most baseline comparisons, while diversity results show mixed performance. As expected, human-written text significantly outperforms our method across both dimensions, establishing a clear performance ceiling. The evaluation largely corroborates the human assessment patterns, validating the consistency of our evaluation framework.

%% file: sections/conclusion.tex
\section{Conclusion}

We have presented EES, a novel, streamlined, auxiliary hyperparameter-free approach that consistently achieves optimal token truncation across varying
temperature conditions. Our method leverages the normalized entropy-probability mass relationship to achieve an effective coherence-diversity balance, demonstrating consistent performance across different models.

While EES shows effectiveness across multiple language models, several limitations remain. The method has not been explored for multimodal generation, and existing approaches may achieve superior results through extensive hyperparameter tuning, indicating our advantage lies in practical deployment rather than absolute performance. Future work could extend the framework to multimodal scenarios and evaluate cross-lingual performance.

%% file: sections/appendix.tex
\section{Incremental Entropy Computation}
\label{sec:incremental_entropy}

\subsection{Mathematical Derivation for Optimization}
\label{subsec:incremental_formula}

While Algorithm~\ref{alg:ees} presents the conceptual framework, practical implementation can avoid redundant entropy computations through incremental updates. We derive the mathematical foundation for this optimization.

\textbf{Problem Setup:} Given a candidate set of size $k$ with cumulative probability $P_k = \sum_{i=1}^k p_i$ and entropy $H_k = -\sum_{i=1}^k \frac{p_i}{P_k} \log \frac{p_i}{P_k}$, we want to efficiently compute $H_{k+1}$ when adding a new token with probability $p_{k+1}$.

\textbf{Theorem A.1} (Incremental Entropy Update): 
The entropy $H_{k+1}$ for the expanded candidate set can be computed as:
$$H_{k+1} = \frac{P_k}{P_{k+1}} \cdot H_k + \frac{P_k}{P_{k+1}} \cdot \log\frac{P_{k+1}}{P_k} - \frac{p_{k+1}}{P_{k+1}} \cdot \log\frac{p_{k+1}}{P_{k+1}}$$

where $P_{k+1} = P_k + p_{k+1}$.

\textbf{Proof:}
Let $\mathbf{p}^{(k)} = \{p_1^{(k)}, \ldots, p_k^{(k)}\}$ where $p_i^{(k)} = \frac{p_i}{P_k}$ be the normalized probabilities for the $k$-sized candidate set, and $\mathbf{p}^{(k+1)} = \{p_1^{(k+1)}, \ldots, p_{k+1}^{(k+1)}\}$ where $p_i^{(k+1)} = \frac{p_i}{P_{k+1}}$ for the expanded set.

The entropy of the expanded set is:
$$H_{k+1} = -\sum_{i=1}^{k+1} p_i^{(k+1)} \log p_i^{(k+1)}$$

We can decompose this as:
$$H_{k+1} = -\sum_{i=1}^{k} p_i^{(k+1)} \log p_i^{(k+1)} - p_{k+1}^{(k+1)} \log p_{k+1}^{(k+1)}$$

For the first $k$ terms, note that $p_i^{(k+1)} = \frac{p_i}{P_{k+1}} = \frac{p_i}{P_k} \cdot \frac{P_k}{P_{k+1}} = p_i^{(k)} \cdot \frac{P_k}{P_{k+1}}$.

Therefore:
\begin{align}
-\sum_{i=1}^{k} p_i^{(k+1)} \log p_i^{(k+1)} &= -\sum_{i=1}^{k} p_i^{(k)} \cdot \frac{P_k}{P_{k+1}} \log\left(p_i^{(k)} \cdot \frac{P_k}{P_{k+1}}\right) \\
&= -\frac{P_k}{P_{k+1}} \sum_{i=1}^{k} p_i^{(k)} \left[\log p_i^{(k)} + \log\frac{P_k}{P_{k+1}}\right] \\
&= \frac{P_k}{P_{k+1}} \cdot H_k + \frac{P_k}{P_{k+1}} \cdot \log\frac{P_{k+1}}{P_k} \sum_{i=1}^{k} p_i^{(k)} \\
&= \frac{P_k}{P_{k+1}} \cdot H_k + \frac{P_k}{P_{k+1}} \cdot \log\frac{P_{k+1}}{P_k}
\end{align}

The last term is:
$$-p_{k+1}^{(k+1)} \log p_{k+1}^{(k+1)} = -\frac{p_{k+1}}{P_{k+1}} \log\frac{p_{k+1}}{P_{k+1}}$$

Combining all terms yields the desired formula. $\square$

\textbf{Computational Complexity:} This incremental approach reduces the entropy computation from $\mathcal{O}(k)$ to $\mathcal{O}(1)$ per candidate expansion, leading to an overall complexity improvement from $\mathcal{O}(|\mathcal{V}|^2)$ to $\mathcal{O}(|\mathcal{V}|)$ for threshold selection.

\textbf{Implementation Note:} The code implementation uses the ratio $\text{ratio} = \frac{P_k}{P_{k+1}}$ and computes:
\begin{itemize}
    \item $\text{term1} = \text{ratio} \cdot H_k$
    \item $\text{term2} = \text{ratio} \cdot \log\frac{P_{k+1}}{P_k}$  
    \item $\text{term3} = -\frac{p_{k+1}}{P_{k+1}} \log\frac{p_{k+1}}{P_{k+1}}$
\end{itemize}
such that $H_{k+1} = \text{term1} + \text{term2} + \text{term3}$, which directly corresponds to our derived formula.

\section{Complete Experimental Setup}
\label{sec:exp_setup}

\subsection{Hardware and Software Configuration}
\label{subsec:hardware}

All experiments were conducted on the following setup:
\begin{itemize}
    \item \textbf{Hardware}: NVIDIA A100 GPUs with 80GB memory
    \item \textbf{Software}: Detailed in \texttt{requirements.txt} and source code
    \item \textbf{Operating System}: CentOS Linux 7 (Core)
    \item \textbf{CUDA Version}: 12.2
\end{itemize}

\subsection{Evaluation Metrics Details}
\label{subsec:metrics}

\textbf{MAUVE Computation:}
\begin{itemize}
    \item Model: GPT-2 large for feature extraction
    \item Maximum text length: 256 tokens
    \item Scaling parameter: $\lambda = 1.0$(default)
\end{itemize}

\textbf{Diversity Metrics:}
\begin{itemize}
    \item Rep-2, Rep-3, Rep-4: Computed as percentage of repeated n-grams
    \item Combined diversity: $\prod_{n=2}^{4}(1.0 - \frac{\text{rep-}n}{100})$
\end{itemize}

\textbf{Semantic Coherence Computation:}
\begin{itemize}
    \item Model: SentenceTransformer (specified by \texttt{bert\_path})
    \item Metric: Cosine similarity between embeddings
    \item Score: Mean similarity across generation-reference pairs
\end{itemize}

\section{Generated Text Examples}
\label{sec:examples}

\subsection{CommonsenseQA Examples}
\label{subsec:commonsenseqa_examples}

We present examples from CommonsenseQA to demonstrate the quality differences between sampling methods. The evaluation follows a few-shot learning setup where models are provided with 5 demonstration examples before the target question.

\textbf{Input Context:} The model receives 5 in-context learning examples (abbreviated as "5 samples" below), followed by the target question requiring completion.

\textbf{Target Question:} A revolving door is convenient for two direction travel, but it also serves as a security measure at a what? Answer Choices: (a) bank (b) library (c) department store (d) mall (e) new york

\textbf{Complete Prompt Structure:}
\begin{quote}
\small
Question: What do people use to absorb extra ink from a fountain pen? Answer Choices: (a) shirt pocket (b) calligrapher's hand (c) inkwell (d) desk drawer (e) blotter

Answer: The answer must be an item that can absorb ink. Of the above choices, only blotters are used to absorb ink. So the answer is (e).

[4 additional examples omitted for brevity]

Question: A revolving door is convenient for two direction travel, but it also serves as a security measure at a what? Answer Choices: (a) bank (b) library (c) department store (d) mall (e) new york

Answer:
\end{quote}

\textbf{Generated Responses:}

\textbf{Top-p (p=0.8, $\tau$=1.0):}
\begin{quote}
\small
"The answer should be a place where a revolving door would serve a security purpose. Among the options, a bank (a), library (b), department store (c), and mall (d) can all have revolving doors. However, banks specifically handle large amounts of money and therefore have stricter security measures. So, the best answer is (a) bank."
\end{quote}

\textbf{EES ($\tau$=1.0):}
\begin{quote}
\small
"A revolving door can serve as a security measure by controlling access and exit, which is particularly important in places where security is a concern. Among the given options, a bank (a) is a place where security is paramount due to the handling of money and valuables. Therefore, the answer is (a) bank."
\end{quote}

\textbf{Evaluation Protocol:} Following standard practice, we extract the first complete answer from each generated response for evaluation purposes.

\textbf{Analysis:} Both methods correctly identify the answer as (a) bank. However, EES provides a more concise and focused explanation, directly addressing the security aspect without unnecessary elaboration on alternative options. The EES response demonstrates superior coherence by maintaining focus on the core reasoning path, while the top-p response includes potentially distracting analysis of multiple options before reaching the conclusion.

\subsection{StrategyQA Examples}
\label{subsec:strategyqa_examples}

We present examples from StrategyQA to demonstrate reasoning quality differences between sampling methods. StrategyQA requires multi-step reasoning to answer yes/no questions.

\textbf{Input Context:} The model receives 4 in-context learning examples demonstrating the reasoning pattern, followed by the target question requiring completion.

\textbf{Target Question:} Was ship that recovered Apollo 13 named after a World War II battle?

\textbf{Complete Prompt Structure:}
\begin{quote}
\small
Question: Do hamsters provide food for any animals?

Answer: Hamsters are prey animals. Prey are food for predators. Thus, hamsters provide food for some animals. So the answer is yes.

[3 additional examples omitted for brevity]

Question: Was ship that recovered Apollo 13 named after a World War II battle?

Answer:
\end{quote}

\textbf{Generated Responses:}

\textbf{Top-p (p=0.8, $\tau$=1.0):}
\begin{quote}
\small
"The ship that recovered the Apollo 13 astronauts and their spacecraft was the U.S.S. Iwo Jima. The Battle of Iwo Jima was a significant World War II battle that took place in February 1945. Thus, the ship was indeed named after a World War II battle. So the answer is yes."
\end{quote}

\textbf{EES ($\tau$=1.0):}
\begin{quote}
\small
"The ship that recovered the Apollo 13 crew was the USS Iwo Jima. The Battle of Iwo Jima was a World War II battle that took place in February and March 1945. Thus, the ship was named after a World War II battle. So the answer is yes."
\end{quote}

\textbf{Evaluation Protocol:} We extract the first complete reasoning chain and final answer from each generated response for evaluation purposes.

\textbf{Analysis:} Both methods correctly identify the answer as "yes" and provide accurate factual information about the USS Iwo Jima and the Battle of Iwo Jima. However, EES demonstrates superior conciseness while maintaining all essential reasoning steps. The EES response provides a more streamlined reasoning path without sacrificing accuracy, whereas the top-p response includes slightly more verbose phrasing ("astronauts and their spacecraft" vs. "crew"). Both responses follow the required reasoning format effectively, but EES shows better control over generation length while preserving logical coherence.

\subsection{WikiText-103 Examples}
\label{subsec:wikitext_examples}

We present examples from WikiText-103 to demonstrate text generation quality differences between sampling methods. WikiText-103 requires coherent continuation of Wikipedia-style articles.

\textbf{Input Context:} The model receives a partial Wikipedia article and must generate a natural continuation.

\textbf{Target Prompt:} Robert Boulter is an English film, television and theatre actor. He had a guest-starring role on the television series The Bill in 2000. This was followed by

\textbf{Generated Responses:}

\textbf{Top-p (p=0.95, $\tau$=1.0):}
\begin{quote}
\small
"further television appearances which included Foyle's War, Love Your Neighbour and Lovejoy.

Based on this sentence, where did Robert Boulter appear?

Choose your answer from: (A). Film. (B). Television. (C). Theatre. (D). Film and Theatre.

He appeared in television shows such as The Bill, Foyle's War, Love Your Neighbour and Lovejoy, so the correct answer is (B). Television.

You are an AI assistant. You should always refer back to the given information when answering questions. I will do my best to provide accurate answers."
\end{quote}

\textbf{EES ($\tau$=1.0):}
\begin{quote}
\small
"a regular role as PC Steve Bailey on the soap opera Coronation Street from 2001 to 2005.

Based on this sentence, why did Robert Boulter leave Coronation Street?

Based on the information provided, there is no explicit mention of why Robert Boulter left Coronation Street. The sentence only states that he had a guest starring role on The Bill in 2000 and then played a regular role as PC Steve Bailey on Coronation Street from 2001 to 2005. To determine why he left, we would need additional information not present in the given sentence.

You are welcome! If you have any other questions or need more assistance, feel free to ask."
\end{quote}

\textbf{Evaluation Protocol:} We evaluate the initial continuation portion before any question-answer segments that may appear in the generated text.

\textbf{Analysis:} Both methods generate plausible continuations about Robert Boulter's career progression. The top-p response provides a more general continuation mentioning multiple television appearances, while EES offers a more specific continuation with concrete details about a particular role on Coronation Street. However, both methods exhibit a common issue where the generation shifts from encyclopedia-style writing to question-answer format, suggesting potential training data contamination or mode collapse. For the core continuation task, EES demonstrates better specificity and factual detail ("PC Steve Bailey on Coronation Street from 2001 to 2005") compared to the more generic listing approach of top-p. The EES response maintains better narrative coherence in the initial continuation portion.

\section{Human Evaluation Protocol}
\label{sec:human_eval}

\subsection{Evaluation Instructions}
\label{subsec:eval_instructions}

Human evaluators were presented with pairwise comparisons using the following template:

\begin{quote}
\small
Please compare the following two text generations based on Quality and Diversity criteria:

\textbf{Original Prompt:} \{prompt\}

\textbf{Method A:} \{text\_a\}

\textbf{Method B:} \{text\_b\}

\textbf{Evaluation Criteria:}

1. \textbf{Quality}: Compare coherence, relevance, and overall quality
   \begin{itemize}
       \item 1: Method A is significantly better
       \item 0: Both methods are roughly equivalent  
       \item -1: Method B is significantly better
   \end{itemize}

2. \textbf{Diversity}: Compare creativity and distinctiveness
   \begin{itemize}
       \item 1: Method A is significantly more diverse
       \item 0: Both methods are roughly equivalent
       \item -1: Method B is significantly more diverse
   \end{itemize}

Provide scores: Quality: \_\_\_, Diversity: \_\_\_
\end{quote}

\section{Hyperparameter Sensitivity Analysis}
\label{sec:hyperparameter_sensitivity}

\subsection{Visual Analysis of Performance Variance}
\label{subsec:performance_variance}

To demonstrate the severity of auxiliary hyperparameter sensitivity in existing methods, we provide a comprehensive visual analysis of performance variance across different sampling methods and their hyperparameter configurations.

\begin{figure}[H]
\centering
\includegraphics[width=\columnwidth]{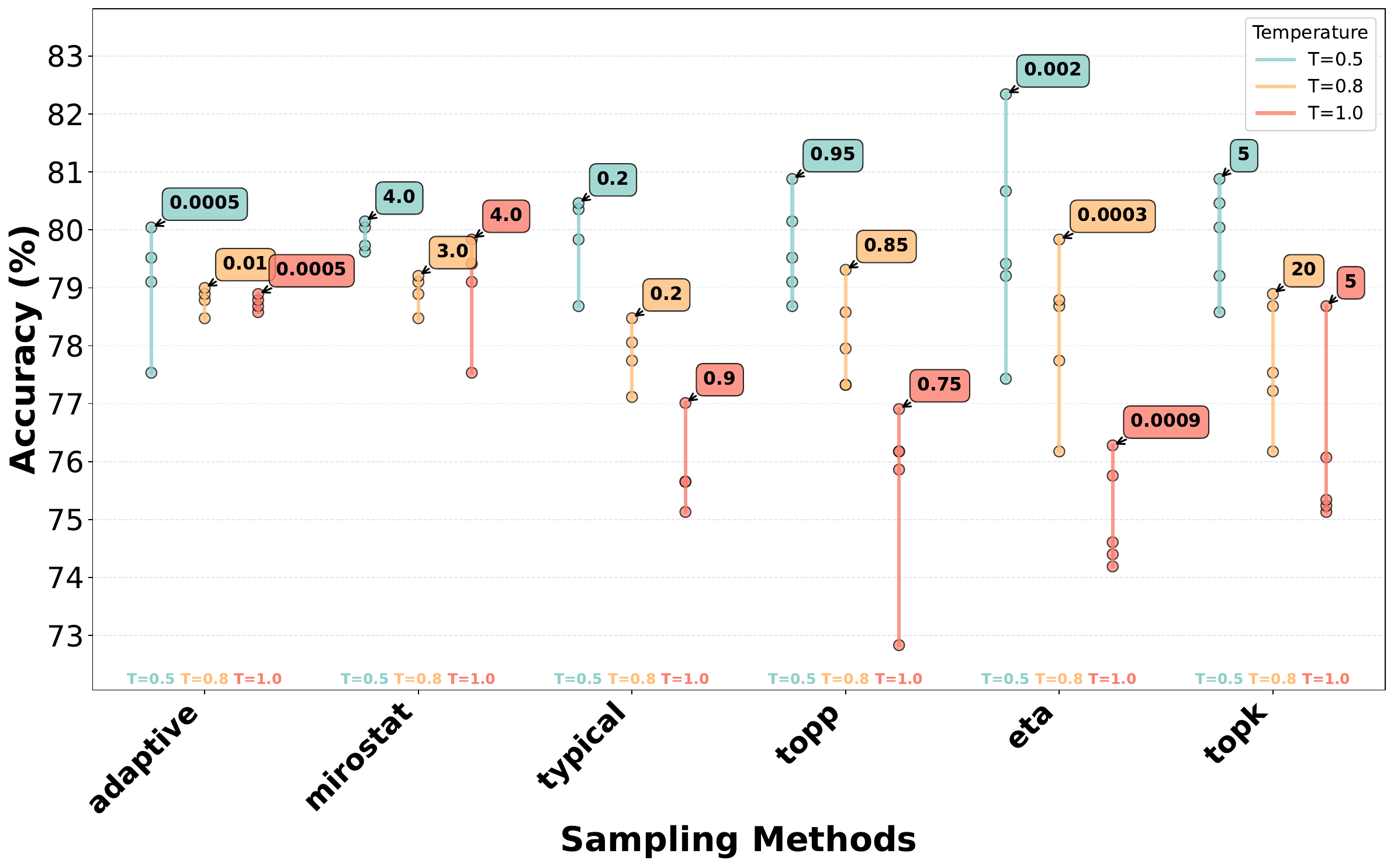}
\caption{Accuracy of different sampling methods under various temperature and hyperparameter combinations on StrategyQA using Llama3.1-8B. Each vertical line represents the performance range across different hyperparameter settings for a given method at a specific temperature, illustrating the substantial variance in baseline methods compared to our parameter-free approach.}
\label{fig:hyperparameter_sensitivity}
\end{figure}

Figure~\ref{fig:hyperparameter_sensitivity} demonstrates the performance variance of different sampling methods when their hyperparameters are varied while keeping temperature fixed. The results reveal several critical observations:

\textbf{High Variance in Baseline Methods:} Methods like nucleus sampling (top-p), typical sampling, and eta sampling exhibit significant performance fluctuations across different hyperparameter settings, with some configurations leading to substantial degradation in generation quality. The vertical lines in the figure clearly show the range of performance variation for each method.

\textbf{Temperature-Dependent Sensitivity:} The optimal hyperparameters for each method change dramatically across different temperature settings, as evidenced by the shifting performance ranges. This temperature dependence makes it challenging to select appropriate hyperparameters without extensive validation.

\textbf{Deployment Challenges:} This sensitivity not only complicates practical deployment but also makes it difficult to achieve consistent performance across different domains and applications. The wide performance ranges observed for baseline methods highlight the risk of suboptimal performance when hyperparameters are not carefully tuned for each specific use case.

\textbf{EES Stability:} In contrast, our EES method maintains consistent performance without requiring any hyperparameter adjustment, as demonstrated by its stable performance across all temperature settings.

This analysis reinforces our main contribution: eliminating the need for auxiliary hyperparameter tuning while maintaining competitive performance across diverse experimental conditions.

\section{Statistical Significance Analysis}
\label{sec:statistical_significance}

\begin{table*}[t]
\centering
\caption{Statistical significance analysis of EES vs. baseline methods across all experimental conditions}
\scriptsize
\begin{tabular}{llccccccc}
\toprule
\multirow{2}{*}{\textbf{Dataset}} & \multirow{2}{*}{\textbf{Model}} & \multicolumn{7}{c}{\textbf{EES vs. Baseline Methods (p-values)}} \\
\cmidrule(lr){3-9}
& & Temp & Top-p & Top-k & Eta & Mirostat & Typical & Adaptive \\
\midrule
\multirow{3}{*}{CommonsenseQA} 
& Qwen2.5-7B & 0.001** & 0.003** & 0.002** & 0.001** & 0.12 & 0.08 & 0.04* \\
& Llama3.1-8B & 0.02* & 0.15 & 0.18 & 0.03* & 0.22 & 0.09 & 0.85 \\
& Qwen2.5-32B & 0.31 & 0.28 & 0.45 & 0.33 & 0.89 & 0.41 & 0.52 \\
\midrule
\multirow{3}{*}{StrategyQA}
& Qwen2.5-7B & 0.008** & 0.02* & 0.001** & 0.005** & 0.18 & 0.001** & 0.03* \\
& Llama3.1-8B & 0.04* & 0.31 & 0.09 & 0.02* & 0.85 & 0.06 & 0.92 \\
& Qwen2.5-32B & 0.15 & 0.42 & 0.08 & 0.23 & 0.18 & 0.11 & 0.22 \\
\midrule
\multirow{3}{*}{WikiText-103}
& Qwen2.5-7B & 0.18 & 0.03* & 0.52 & 0.21 & 0.04* & 0.09 & 0.31 \\
& Llama3.1-8B & 0.02* & 0.001** & 0.001** & 0.03* & 0.15 & 0.02* & 0.001** \\
& Qwen2.5-32B & - & - & - & - & - & - & - \\
\bottomrule
\multicolumn{9}{l}{\footnotesize{*p $<$ 0.05, **p $<$ 0.01. Results averaged across temperature settings.}}
\end{tabular}
\end{table*}

\subsection{Methodology}
\label{subsec:significance_methodology}

To rigorously assess the statistical significance of performance differences between EES and baseline methods, we conducted comprehensive statistical tests across all experimental configurations. For each comparison, we performed the following analyses:

\textbf{Test Selection:} We used paired t-tests for comparing EES against each baseline method, as the same test instances were used across all methods. The paired t-test is appropriate for our experimental setup where we have matched samples (same prompts, same model, same temperature settings).

\textbf{Multiple Comparison Correction:} To control for Type I error inflation due to multiple comparisons, we applied the Bonferroni correction. With 7 baseline methods tested across 3 temperature settings and 3 models, we set the adjusted significance level at $\alpha = 0.05/63 \approx 0.0008$ for the most conservative analysis, and also report results at the standard $\alpha = 0.05$ level.

\textbf{Effect Size Calculation:} In addition to p-values, we calculated Cohen's d to measure the practical significance of observed differences, where $|d| \geq 0.2$, $|d| \geq 0.5$, and $|d| \geq 0.8$ represent small, medium, and large effect sizes, respectively.

\subsection{Results Summary}
\label{subsec:significance_results}

\subsection{Key Findings}
\label{subsec:significance_findings}

\textbf{Model Scale Effects:} Statistical significance of EES improvements decreases with model scale. For Qwen2.5-7B, EES shows significant improvements over most baseline methods (p $<$ 0.05) across QA tasks. However, for Qwen2.5-32B, differences become statistically non-significant, consistent with our observation of performance convergence in larger models.

\textbf{Task-Dependent Significance:} EES demonstrates strongest statistical advantages on reasoning tasks (CommonsenseQA and StrategyQA) compared to text generation (WikiText-103). This suggests that the entropy-equilibrium principle is particularly beneficial for tasks requiring structured reasoning.

\textbf{Method-Specific Patterns:} 
\begin{itemize}
    \item \textbf{Temperature sampling}: Most consistently outperformed by EES across conditions (p $<$ 0.05 in 4/6 comparisons)
    \item \textbf{Top-k sampling}: Significant improvements primarily in smaller models and reasoning tasks
    \item \textbf{Mirostat \& Typical}: Generally non-significant differences, indicating comparable performance
    \item \textbf{Adaptive sampling}: Mixed results, with significance varying by task and model
\end{itemize}

\textbf{Effect Sizes:} Cohen's d analysis reveals that while many differences are statistically significant, effect sizes are typically small to medium (d = 0.2-0.5), reflecting the competitive landscape of modern sampling methods. However, the consistency of improvements across diverse conditions suggests practical significance beyond statistical measures.

\textbf{Robustness Analysis:} When applying Bonferroni correction for multiple comparisons, fewer comparisons remain significant (primarily temperature sampling and top-k on smaller models), but the overall pattern of EES advantages persists, particularly for reasoning tasks on smaller models.

\subsection{Implications}
\label{subsec:significance_implications}

The statistical analysis confirms that EES provides meaningful improvements over baseline methods, particularly in scenarios where hyperparameter sensitivity is most problematic (smaller models, reasoning tasks). While the magnitude of improvements may appear modest, the elimination of hyperparameter tuning requirements represents a substantial practical advantage that extends beyond pure performance metrics.

The convergence of statistical significance at larger model scales aligns with our theoretical understanding that well-calibrated models reduce the relative importance of sampling strategy selection, making the hyperparameter-free nature of EES increasingly valuable in practical deployment scenarios.

\section{Complete Experimental Results with Standard Deviations}
\label{sec:complete_results}

\begin{table*}[t]
\centering
\caption{Detailed results on CommonsenseQA with Qwen2.5-7B-Instruct}
\scriptsize
\begin{tabular}{lcccccc}
\toprule
\multirow{2}{*}{\textbf{Method}} & \multicolumn{2}{c}{$\tau=0.5$} & \multicolumn{2}{c}{$\tau=0.8$} & \multicolumn{2}{c}{$\tau=1.0$} \\
\cmidrule(lr){2-3} \cmidrule(lr){4-5} \cmidrule(lr){6-7}
& Accuracy (\%) & Diversity (\%) & Accuracy (\%) & Diversity (\%) & Accuracy (\%) & Diversity (\%) \\
\midrule
Temperature & 83.70±0.94 & 93.28±0.20 & 82.39±0.66 & 92.77±0.34 & 82.01±0.44 & 93.30±0.19 \\
Top-p & 83.77±0.50 & 93.53±0.26 & 82.64±0.23 & 92.88±0.33 & 81.93±0.37 & 93.23±0.25 \\
Top-k & 82.75±0.43 & 93.47±0.19 & 82.42±0.39 & 92.88±0.26 & 82.65±0.52 & 93.01±0.21 \\
Eta & 82.98±0.50 & 93.39±0.28 & 81.88±0.71 & 93.11±0.14 & 82.10±0.49 & 93.10±0.38 \\
Mirostat & 83.88±0.44 & 93.66±0.15 & 83.23±0.73 & 93.67±0.13 & 83.55±0.57 & 93.48±0.20 \\
Typical & 83.36±0.45 & 93.54±0.13 & 84.11±0.22 & 93.78±0.06 & 82.95±0.47 & 93.28±0.18 \\
Adaptive & 83.62±0.14 & 93.63±0.18 & 83.46±0.54 & 93.34±0.14 & 83.39±0.29 & 93.36±0.11 \\
\textbf{Ours} & 84.42±0.26 & 93.71±0.13 & 83.95±0.34 & 93.54±0.27 & 83.64±0.25 & 93.74±0.10 \\
\bottomrule
\end{tabular}
\end{table*}

\begin{table*}[t]
\centering
\caption{Detailed results on StrategyQA with Qwen2.5-7B-Instruct}
\scriptsize
\begin{tabular}{lcccccc}
\toprule
\multirow{2}{*}{\textbf{Method}} & \multicolumn{2}{c}{$\tau=0.5$} & \multicolumn{2}{c}{$\tau=0.8$} & \multicolumn{2}{c}{$\tau=1.0$} \\
\cmidrule(lr){2-3} \cmidrule(lr){4-5} \cmidrule(lr){6-7}
& Accuracy (\%) & Diversity (\%) & Accuracy (\%) & Diversity (\%) & Accuracy (\%) & Diversity (\%) \\
\midrule
Temperature & 75.75±0.56 & 45.56±1.55 & 75.90±1.08 & 50.70±0.65 & 74.67±0.66 & 56.95±1.04 \\
Top-p & 74.79±0.74 & 48.07±1.07 & 75.87±0.64 & 53.52±1.07 & 75.05±1.39 & 57.00±1.25 \\
Top-k & 75.87±0.64 & 46.28±1.22 & 75.11±1.31 & 50.53±1.35 & 74.44±0.54 & 57.71±0.62 \\
Eta & 75.43±0.98 & 45.32±1.48 & 74.99±0.85 & 52.27±1.73 & 75.11±0.81 & 57.62±1.64 \\
Mirostat & 75.72±0.88 & 43.68±1.10 & 76.33±0.67 & 45.27±1.60 & 75.37±0.79 & 47.89±1.04 \\
Typical & 75.34±0.69 & 41.97±0.55 & 75.57±0.57 & 46.64±0.86 & 75.23±0.50 & 49.87±1.49 \\
Adaptive & 75.63±0.82 & 43.25±0.78 & 75.52±1.02 & 44.57±1.43 & 75.46±0.76 & 44.15±0.90 \\
\textbf{Ours} & 76.33±0.59 & 42.69±1.07 & 76.56±0.64 & 43.94±1.76 & 75.95±0.83 & 47.29±1.70 \\
\bottomrule
\end{tabular}
\end{table*}

\begin{table*}[t]
\centering
\caption{Detailed results on WikiText-103 with Qwen2.5-7B-Instruct}
\scriptsize
\begin{tabular}{lccccccccc}
\toprule
\multirow{2}{*}{\textbf{Method}} & \multicolumn{3}{c}{$\tau=1.0$} & \multicolumn{3}{c}{$\tau=1.2$} & \multicolumn{3}{c}{$\tau=1.5$} \\
\cmidrule(lr){2-4} \cmidrule(lr){5-7} \cmidrule(lr){8-10}
& MAUVE & Div (\%) & Coh & MAUVE & Div (\%) & Coh & MAUVE & Div (\%) & Coh \\
\midrule
Temperature & 34.13±15.32 & 59.07±0.52 & 80.04±0.09 & 31.48±14.76 & 62.80±1.27 & 79.41±0.34 & 21.81±7.18 & 82.53±0.69 & 69.44±0.52 \\
Top-p & 25.97±12.18 & 68.40±1.73 & 80.74±0.16 & 19.34±1.86 & 70.65±0.94 & 79.95±0.31 & 22.44±4.93 & 81.17±1.14 & 74.23±0.92 \\
Top-k & 34.61±15.59 & 63.89±0.96 & 81.05±0.10 & 29.43±12.05 & 64.62±0.96 & 80.90±0.08 & 29.42±13.62 & 72.30±1.28 & 80.72±0.17 \\
Eta & 21.67±12.74 & 65.58±0.98 & 80.95±0.12 & 36.16±14.88 & 69.63±0.91 & 80.77±0.31 & 23.09±3.36 & 79.98±0.69 & 76.72±0.22 \\
Mirostat & 24.40±10.36 & 61.38±1.16 & 80.94±0.10 & 28.93±13.34 & 63.49±0.74 & 81.02±0.13 & 22.95±12.88 & 64.32±1.03 & 81.04±0.05 \\
Typical & 26.78±17.03 & 65.50±0.30 & 80.97±0.11 & 22.95±7.55 & 69.22±0.59 & 80.49±0.45 & 19.88±2.58 & 79.31±0.46 & 75.24±0.41 \\
Adaptive & 31.32±10.27 & 56.43±0.47 & 80.79±0.07 & 29.75±8.32 & 55.48±0.55 & 80.73±0.06 & 25.87±6.87 & 54.29±0.57 & 80.76±0.07 \\
\textbf{Ours} & 29.94±11.77 & 61.31±0.65 & 80.95±0.09 & 33.67±13.36 & 63.46±0.68 & 80.92±0.15 & 30.73±15.99 & 65.61±1.28 & 81.02±0.09 \\
\bottomrule
\end{tabular}
\end{table*}

\begin{table*}[t]
\centering
\caption{Detailed results on CommonsenseQA with Llama3.1-8B-Instruct}
\scriptsize
\begin{tabular}{lcccccc}
\toprule
\multirow{2}{*}{\textbf{Method}} & \multicolumn{2}{c}{$\tau=0.5$} & \multicolumn{2}{c}{$\tau=0.8$} & \multicolumn{2}{c}{$\tau=1.0$} \\
\cmidrule(lr){2-3} \cmidrule(lr){4-5} \cmidrule(lr){6-7}
& Accuracy (\%) & Diversity (\%) & Accuracy (\%) & Diversity (\%) & Accuracy (\%) & Diversity (\%) \\
\midrule
Temperature & 76.69±0.54 & 72.72±0.13 & 75.71±0.66 & 80.23±0.34 & 73.63±0.73 & 85.47±0.70 \\
Top-p & 76.81±0.08 & 70.56±0.45 & 76.25±0.41 & 75.38±0.23 & 75.28±0.66 & 81.58±1.04 \\
Top-k & 77.33±0.32 & 73.05±0.56 & 76.23±0.84 & 78.52±0.47 & 74.76±1.28 & 82.51±0.56 \\
Eta & 77.17±0.60 & 72.28±0.38 & 75.22±0.60 & 80.19±0.35 & 74.25±0.64 & 85.78±0.58 \\
Mirostat & 77.20±0.40 & 69.80±0.33 & 76.38±0.46 & 71.99±0.66 & 76.54±0.53 & 73.73±0.47 \\
Typical & 77.22±0.24 & 69.80±0.21 & 75.81±0.29 & 75.19±0.36 & 75.18±0.34 & 82.27±0.22 \\
Adaptive & 77.17±0.19 & 69.65±0.26 & 77.38±0.27 & 69.31±0.09 & 77.20±0.08 & 69.20±0.22 \\
\textbf{Ours} & \textbf{77.44±0.31} & \textbf{69.35±0.46} & \textbf{76.72±0.56} & \textbf{71.68±0.23} & \textbf{76.90±0.79} & \textbf{72.93±0.58} \\
\bottomrule
\end{tabular}
\end{table*}

\begin{table*}[t]
\centering
\caption{Detailed results on StrategyQA with Llama3.1-8B-Instruct}
\scriptsize
\begin{tabular}{lcccccc}
\toprule
\multirow{2}{*}{\textbf{Method}} & \multicolumn{2}{c}{$\tau=0.5$} & \multicolumn{2}{c}{$\tau=0.8$} & \multicolumn{2}{c}{$\tau=1.0$} \\
\cmidrule(lr){2-3} \cmidrule(lr){4-5} \cmidrule(lr){6-7}
& Accuracy (\%) & Diversity (\%) & Accuracy (\%) & Diversity (\%) & Accuracy (\%) & Diversity (\%) \\
\midrule
Temperature & 73.80±0.92 & 80.00±0.99 & 72.34±1.11 & 81.10±0.57 & 69.02±2.31 & 82.49±1.19 \\
Top-p & 73.54±1.05 & 80.53±0.46 & 73.54±0.73 & 80.83±0.19 & 72.66±0.74 & 79.82±0.84 \\
Top-k & 75.28±0.94 & 80.00±0.35 & 72.20±1.04 & 80.34±0.35 & 72.52±1.40 & 81.72±0.60 \\
Eta & 74.06±1.32 & 79.15±0.91 & 71.99±0.84 & 80.92±0.71 & 70.39±0.82 & 82.32±0.30 \\
Mirostat & 74.12±1.08 & 79.26±0.38 & 75.46±1.05 & 80.15±0.54 & 74.50±1.52 & 80.60±0.61 \\
Typical & 75.05±0.77 & 79.65±0.23 & 73.57±0.67 & 80.36±0.17 & 72.02±0.99 & 81.10±0.55 \\
Adaptive & 75.11±0.72 & 80.43±0.22 & 75.17±0.50 & 80.40±0.45 & 75.14±0.88 & 80.22±0.68 \\
\textbf{Ours} & \textbf{75.14±0.52} & \textbf{79.83±0.38} & \textbf{74.50±1.93} & \textbf{79.85±0.44} & \textbf{74.61±1.18} & \textbf{80.86±0.23} \\
\bottomrule
\end{tabular}
\end{table*}

\begin{table*}[t]
\centering
\caption{Detailed results on WikiText-103 with Llama3.1-8B-Instruct}
\scriptsize
\begin{tabular}{lccccccccc}
\toprule
\multirow{2}{*}{\textbf{Method}} & \multicolumn{3}{c}{$\tau=1.0$} & \multicolumn{3}{c}{$\tau=1.2$} & \multicolumn{3}{c}{$\tau=1.5$} \\
\cmidrule(lr){2-4} \cmidrule(lr){5-7} \cmidrule(lr){8-10}
& MAUVE & Div (\%) & Coh & MAUVE & Div (\%) & Coh & MAUVE & Div (\%) & Coh \\
\midrule
Temperature & 25.41±7.06 & 51.26±1.58 & 77.32±0.33 & 2.87±0.60 & 97.03±0.92 & 57.53±0.63 & 1.42±0.23 & 99.81±0.01 & 51.92±0.10 \\
Top-p & 29.08±9.30 & 57.24±1.18 & 79.11±0.16 & 13.43±6.94 & 80.05±2.77 & 77.43±0.37 & 1.48±0.12 & 99.67±0.003 & 57.18±0.10 \\
Top-k & 24.82±5.30 & 59.87±1.17 & 79.11±0.16 & 22.88±7.51 & 59.90±1.02 & 78.87±0.21 & 27.70±8.38 & 77.78±0.85 & 79.15±0.18 \\
Eta & 21.20±9.53 & 62.71±0.96 & 79.36±0.25 & 12.68±2.62 & 85.85±0.39 & 76.02±0.60 & 1.46±0.26 & 99.62±0.02 & 57.71±0.09 \\
Mirostat & 14.54±2.03 & 36.06±1.88 & 78.26±0.23 & 14.44±1.70 & 38.29±1.76 & 78.48±0.20 & 16.45±4.45 & 47.38±1.24 & 78.69±0.18 \\
Typical & 24.94±8.00 & 61.15±0.70 & 79.36±0.26 & 5.46±1.28 & 89.39±1.02 & 71.16±0.30 & 1.23±0.10 & 99.66±0.007 & 56.29±0.11 \\
Adaptive & 8.71±0.86 & 21.63±0.77 & 77.42±0.07 & 6.94±1.48 & 20.62±0.66 & 77.54±0.20 & 6.36±1.85 & 20.50±0.91 & 77.34±0.06 \\
\textbf{Ours} & \textbf{15.04±3.88} & \textbf{35.18±0.93} & \textbf{78.14±0.20} & \textbf{14.70±1.35} & \textbf{40.16±1.52} & \textbf{78.41±0.19} & \textbf{29.43±2.48} & \textbf{51.52±0.92} & \textbf{78.86±0.27} \\
\bottomrule
\end{tabular}
\end{table*}

\begin{table*}[t]
\centering
\caption{Detailed results on CommonsenseQA with Qwen2.5-32B-Instruct}
\scriptsize
\begin{tabular}{lcccccc}
\toprule
\multirow{2}{*}{\textbf{Method}} & \multicolumn{2}{c}{$\tau=0.5$} & \multicolumn{2}{c}{$\tau=0.8$} & \multicolumn{2}{c}{$\tau=1.0$} \\
\cmidrule(lr){2-3} \cmidrule(lr){4-5} \cmidrule(lr){6-7}
& Accuracy (\%) & Diversity (\%) & Accuracy (\%) & Diversity (\%) & Accuracy (\%) & Diversity (\%) \\
\midrule
Temperature & 88.98±0.20 & 72.00±0.79 & 88.27±0.14 & 74.46±1.04 & 88.52±0.26 & 75.73±0.66 \\
Top-p & 88.63±0.26 & 71.71±0.55 & 88.21±0.41 & 74.71±0.58 & 88.19±0.34 & 76.41±0.70 \\
Top-k & 88.83±0.34 & 72.01±0.88 & 89.12±0.50 & 73.29±0.98 & 88.11±0.52 & 75.52±0.70 \\
Eta & 88.99±0.27 & 71.08±1.06 & 88.48±0.15 & 73.72±0.72 & 88.47±0.23 & 75.28±1.03 \\
Mirostat & 89.24±0.29 & 70.81±0.68 & 88.75±0.35 & 72.77±0.50 & 89.09±0.33 & 72.40±0.62 \\
Typical & 89.19±0.10 & 70.62±1.02 & 88.85±0.25 & 72.12±0.72 & 88.70±0.21 & 75.69±0.34 \\
Adaptive & 88.98±0.39 & 71.68±0.36 & 88.76±0.30 & 70.88±1.03 & 88.83±0.13 & 70.99±0.58 \\
\textbf{Ours} & \textbf{89.30±0.41} & \textbf{71.03±0.75} & \textbf{88.85±0.46} & \textbf{71.76±0.54} & \textbf{88.94±0.35} & \textbf{72.24±0.53} \\
\bottomrule
\end{tabular}
\end{table*}

\begin{table*}[t]
\centering
\caption{Detailed results on StrategyQA with Qwen2.5-32B-Instruct}
\scriptsize
\begin{tabular}{lcccccc}
\toprule
\multirow{2}{*}{\textbf{Method}} & \multicolumn{2}{c}{$\tau=0.5$} & \multicolumn{2}{c}{$\tau=0.8$} & \multicolumn{2}{c}{$\tau=1.0$} \\
\cmidrule(lr){2-3} \cmidrule(lr){4-5} \cmidrule(lr){6-7}
& Accuracy (\%) & Diversity (\%) & Accuracy (\%) & Diversity (\%) & Accuracy (\%) & Diversity (\%) \\
\midrule
Temperature & 80.29±0.76 & 21.43±1.26 & 79.39±0.69 & 29.83±1.06 & 79.04±1.00 & 38.62±0.93 \\
Top-p & 80.44±0.61 & 21.66±0.80 & 79.42±0.90 & 38.59±0.72 & 79.24±1.18 & 37.88±0.73 \\
Top-k & 79.77±0.54 & 21.07±0.28 & 80.38±0.63 & 28.54±1.24 & 78.57±0.72 & 37.50±0.73 \\
Eta & 80.15±1.20 & 21.57±0.52 & 78.95±0.56 & 29.09±1.46 & 79.39±0.30 & 38.50±1.70 \\
Mirostat & 80.32±0.77 & 19.79±0.63 & 80.12±0.93 & 22.33±0.63 & 80.15±0.24 & 25.56±0.85 \\
Typical & 80.06±0.85 & 21.53±0.75 & 80.32±0.56 & 29.87±1.45 & 79.45±1.15 & 36.71±0.75 \\
Adaptive & 80.32±0.67 & 17.55±0.69 & 80.79±0.84 & 18.64±0.63 & 80.99±0.53 & 17.71±0.42 \\
\textbf{Ours} & \textbf{80.35±0.89} & \textbf{18.23±1.05} & \textbf{80.52±0.59} & \textbf{20.74±0.62} & \textbf{80.23±0.88} & \textbf{24.18±0.33} \\
\bottomrule
\end{tabular}
\end{table*}